%% file: alt2023-sample.tex
\documentclass[12pt]{alt2023} 

\title[A Unified Algorithm for Stochastic Path Problems]{A Unified Algorithm
for Stochastic Path Problems} 
\usepackage{times}
\input{commands}

\altauthor{%
 \Name{Christoph Dann} \Email{cdann@cdann.net}\\
 \addr Google
 \AND
 \Name{Chen-Yu Wei} \Email{chenyu.wei@usc.edu}\\
 \addr USC 
 \AND
 \Name{Julian Zimmert}
 \Email{zimmert@google.com}\\
 \addr Google
}

\begin{document}

\maketitle

\begin{abstract}%
We study reinforcement learning in stochastic path (SP) problems. The goal in these problems is to maximize the expected sum of rewards until the agent reaches a terminal state. We provide the first regret guarantees in this general problem by analyzing a simple optimistic algorithm. Our regret bound matches the best known results for the well-studied special case of stochastic shortest path (SSP) with all non-positive rewards. For SSP, we present an adaptation procedure for the case when the scale of rewards $B_\star$ is unknown. We show that there is no price for adaptation, and our regret bound matches that with a known $\Bstar$. We also provide a scale adaptation procedure for the special case of stochastic longest paths (SLP) where all rewards are non-negative. However, unlike in SSP, we show through a lower bound that there is an unavoidable price for adaptation. 
\end{abstract}


\input{introduction}
\input{preliminaries}

\input{general}

\input{reward}

\input{cost}
\input{conclusion} 

\bibliography{alt2023-sample.bib}

\appendix

\input{appendix-general}

\input{appendix-reward}

\input{appendix-cost}

\input{appendix-lower-bound}

\input{appendix-auxiliary}
\input{appendix-nonproper}

\end{document}

%% file: commands.tex
\usepackage{xspace}
\usepackage{amsmath}
\usepackage{graphicx}
\usepackage{framed}
\usepackage{bbm}
\usepackage{algorithm}
\usepackage{natbib}
\usepackage{mathrsfs}
\usepackage{amssymb}
\usepackage{bm}
\usepackage{enumitem}
\usepackage{makecell}
\usepackage{tabulary}
\usepackage{xcolor}
\usepackage{prettyref}
\usepackage{mathtools}
\usepackage{amsmath}
\usepackage{multirow}
\usepackage{nicefrac}
\usepackage{amssymb}
\usepackage{pifont}
\definecolor{Green}{rgb}{0.13, 0.65, 0.3}
\usepackage{colortbl}
\newcommand{\KL}{\text{KL}}

\definecolor{plum}{rgb}{0.56, 0.27, 0.52}

\DeclareMathOperator*{\argmin}{argmin} 
\DeclareMathOperator*{\argmax}{argmax} 

\newcommand{\Reg}{\text{\rm Reg}}

\newcommand{\calA}{\mathcal{A}}

\newcommand{\calE}{\mathcal{E}}

\newcommand{\calS}{\mathcal{S}}

\newcommand{\calF}{\mathcal{F}}

\newcommand{\Bstar}{B_\star}

\newcommand{\algComment}[1]{\scalebox{0.9}{\color{blue}\texttt{/* #1 */}}}

\newcommand{\poly}{\text{poly}}
\newcommand{\E}{\mathbb{E}}
\newcommand{\bbP}{\mathbb{P}}

\newcommand{\norm}[1]{\left\|#1\right\|}

\newcommand{\one}{\mathbf{1}}
\newcommand{\ind}[1]{\mathbb{I}\left\{#1\right\}}
\newcommand{\inner}[1]{\langle#1\rangle}

\newcommand{\tiliota}{\tilde{\iota}}

\newcommand{\alg}{{\small\textsf{\textup{ALG}}}\xspace}

\newcommand{\term}{\textbf{term}}

\newcommand{\unif}{\textrm{unif}}

\newtheorem{assumption}{Assumption}
\usepackage{tikz}

\newcommand{\init}{\text{init}}

\newcommand{\nonl}{\renewcommand{\nl}{\let\nl}}

\newcommand{\piHD}{\Pi^{\text{\upshape{HD}}}}
\newcommand{\piSD}{\Pi^{\text{\upshape{SD}}}}
\newcommand{\Vinit}{V_\star}
\newcommand{\Tmax}{T_{\max}}

\newcommand{\Rinit}{R}

\newcommand{\Rmax}{R_{\max}}
\newcommand{\Rstar}{R_\star}
\newcommand{\Tstar}{T_\star}

\newcommand{\tmp}{\text{tmp}}

\usepackage{prettyref}
\newcommand{\pref}[1]{\prettyref{#1}}

\newcommand{\savehyperref}[2]{\texorpdfstring{\hyperref[#1]{#2}}{#2}}
\newrefformat{eq}{\savehyperref{#1}{Eq.~\textup{(\ref*{#1})}}}
\newrefformat{eqn}{\savehyperref{#1}{Equation~\ref*{#1}}}
\newrefformat{lem}{\savehyperref{#1}{Lemma~\ref*{#1}}}
\newrefformat{lemma}{\savehyperref{#1}{Lemma~\ref*{#1}}}
\newrefformat{def}{\savehyperref{#1}{Definition~\ref*{#1}}}
\newrefformat{line}{\savehyperref{#1}{Line~\ref*{#1}}}
\newrefformat{thm}{\savehyperref{#1}{Theorem~\ref*{#1}}}
\newrefformat{corr}{\savehyperref{#1}{Corollary~\ref*{#1}}}
\newrefformat{cor}{\savehyperref{#1}{Corollary~\ref*{#1}}}
\newrefformat{sec}{\savehyperref{#1}{Section~\ref*{#1}}}
\newrefformat{subsec}{\savehyperref{#1}{Section~\ref*{#1}}}
\newrefformat{app}{\savehyperref{#1}{Appendix~\ref*{#1}}}
\newrefformat{assum}{\savehyperref{#1}{Assumption~\ref*{#1}}}
\newrefformat{ex}{\savehyperref{#1}{Example~\ref*{#1}}}
\newrefformat{fig}{\savehyperref{#1}{Figure~\ref*{#1}}}
\newrefformat{alg}{\savehyperref{#1}{Algorithm~\ref*{#1}}}
\newrefformat{rem}{\savehyperref{#1}{Remark~\ref*{#1}}}
\newrefformat{conj}{\savehyperref{#1}{Conjecture~\ref*{#1}}}
\newrefformat{prop}{\savehyperref{#1}{Proposition~\ref*{#1}}}
\newrefformat{proto}{\savehyperref{#1}{Protocol~\ref*{#1}}}
\newrefformat{prob}{\savehyperref{#1}{Problem~\ref*{#1}}}
\newrefformat{claim}{\savehyperref{#1}{Claim~\ref*{#1}}}
\newrefformat{que}{\savehyperref{#1}{Question~\ref*{#1}}}
\newrefformat{op}{\savehyperref{#1}{Open Problem~\ref*{#1}}}
\newrefformat{fn}{\savehyperref{#1}{Footnote~\ref*{#1}}}
\newrefformat{tab}{\savehyperref{#1}{Table~\ref*{#1}}}
\newrefformat{fig}{\savehyperref{#1}{Figure~\ref*{#1}}}
\newrefformat{proc}{\savehyperref{#1}{Procedure~\ref*{#1}}}

\usepackage{caption}


%% file: introduction.tex
\section{Introduction}

Imagine a web application with recurring user visits (epochs).
During each visit, the app can choose from different content to present to the user (actions), which might lead to a desired interaction such as a purchase or click on an ad (reward).
The user's behaviour depends on their internal state, which is influenced by the content provided to them.
Inevitably, at some point the user will abandon the session.

It is natural to model this as an episodic reinforcement-learning problem. 
However, the length of each episode is random and depends on the agent's actions. 
To deal with the random episode length, and hence potentially unbounded cumulative reward in a single episode, one could either consider a fixed horizon problem by clipping the length of each episode, or consider discounted rewards.
Both approaches introduce biases to the actual objective of the agent and we consider a third option: the \emph{stochastic longest path} (SLP) setting, which is analogous to the \emph{stochastic shortest path} (SSP) problem \citep[e.g.,][]{rosenberg2020near, cohen2021minimax, tarbouriech2021stochastic, chen2021implicit} but with positive rewards instead of costs.

In more generality, we can assume a setting in which there are both negative rewards (cost), as well as positive rewards.
For example, users without subscription using the free part of an application might induce an overall negative reward due to the cost of infrastructure.
However, the hope is that the user is convinced by the
free service to upgrade to a subscription, which provides revenue.
We call this setting the \emph{stochastic path} (SP) problem.
We make the following contributions:
\begin{enumerate}
    \item We formalize the general SP problem and provide a simple unified algorithm for it with SSP and SLP as special cases.
    \item We present the first regret upper-bound for the SP and SLP problem and show through lower-bounds that they are minimax-optimal up to log-factors and lower-order terms. Technically, our analysis gives the first near-optimal near-horizon-free regret bound for episodic MDPs when the reward could be positive or negative. In comparison, previous analysis in \cite{zhang2021reinforcement, tarbouriech2021stochastic, chen2021implicit} can only get near-horizon-free bounds for all-non-negative or all-non-positive reward (see \pref{sec: general SP} for more discussions). 
    \item For SSP, when the scale of the sum of rewards $B_\star$ is unknown, we derive an improved procedure to adapt to $\Bstar$. Unlike prior work \citep{tarbouriech2021stochastic, chen2021implicit}, our adaptation procedure allows us to recover the regret bound achieved with a known $B_\star$. 
    \item For SLP, we also derive an algorithm that adapts to unknown $B_\star$. This adaptation is qualitatively different than in the SSP case. In fact, we show through a lower bound that adaptivity to unkonwn $B_\star$ comes at an unavoidable price in SLP.
\end{enumerate}
The contributions 3 and 4 above jointly formalize a distinction between SSP and SLP when the scale of cumulative rewards is unknown. An overview of the main regret bounds derived in this work and a comparison to existing results is available in \pref{tab:summary}.

\renewcommand{\arraystretch}{1.3}
\begin{table}[t]
\small
    \centering
    \caption{Overview of regret bounds for stochastic path problems. See \pref{sec: prelim} for definitions. }\label{tab:summary}
    \begin{tabular}{|c|c|c|c|}
        \hline
        Setting & Scale $B_\star$ &  $\Reg_K$ in $\tilde{O}(\cdot)$ &  \\
        \hline
        \multirow{2}{*}{SP} & \multirow{2}{*}{known} & $R\sqrt{SAK} + \Rmax SA + \Bstar S^2A$ & \pref{thm: main for mixed case} \\
        \cline{3-4}
        & & \makecell{$R\sqrt{SAK}$ \  \textbf{(lower bound)}} & \pref{thm: lower bound for general case}, \pref{thm: impossibility of R*}
        \\
        \hline
        \multirow{4}{*}{SLP} & known & $\sqrt{V_\star B_\star SAK} + B_\star S^2 A$ &  \pref{thm: main thm for reward} \\ 
        \cline{2-4}
        & \multirow{3}{*}{unknown}
        & \makecell{$B_\star S \sqrt{AK} $ or $\sqrt{ V_\star B_\star SAK} + \frac{B_\star^2}{V_\star} S^3 A$}  &   \pref{thm: main for parameter free reward setting}\\
        \cline{3-4}
        & & \makecell{$B_\star  \sqrt{SAK} $ or $\sqrt{ V_\star B_\star SAK} + \frac{B_\star^2}{V_\star} S A$ \\ \textbf{(lower bound)}}  &  \pref{corr: reward setting lower bound for agnostic alg}\\
        \hline
        \multirow{3}{*}{SSP} & known & $\sqrt{V_\star B_\star SAK} + B_\star S^2 A$ & \citet{tarbouriech2021stochastic,chen2021implicit} \\
        \cline{2-4} 
         & 
        \multirow{2}{*}{unknown} & $\sqrt{V_\star B_\star SAK} + B_\star^3 S^3 A$ & \citet{tarbouriech2021stochastic,chen2021implicit}  \\
        \cline{3-4} 
         &  &  
         $\sqrt{V_\star B_\star SAK} + B_\star S^2 A$  &  \pref{thm:ssp_main_result}
         \\
        \hline
    \end{tabular}
\end{table}

\paragraph{Related work}
The SP problem and its special cases SSP and SLP are episodic reinforcement learning settings. When the horizon, the length of each episode is fixed and known, these problems have been extensively studied \citep{dann2015sample,azar2017minimax,jin2018q,efroni2020reinforcement, zanette2019tighter, zhang2020almost}. Among these work on finite-horizon tabular RL, the recent line of work on \emph{horizon-free} algorithms \citep{wang2020reinforcement,zhang2020almost,zhang2022horizon} is of particular interest. These works assume that rewards are non-negative and their cumulative sum are bounded by $1$ and aim for regret that only incurs a logarithmic dependency on the horizon. Although many techniques developed there are useful for our setting as well, the SLP and general SP problem is more difficult since the reward sum is not bounded by $1$, and there are potentially negative rewards. 

The SSP problem has seen a number of publications recently \citep{rosenberg2020near, cohen2021minimax, tarbouriech2020no, tarbouriech2021sample, tarbouriech2021stochastic, vial2022regret,chen2021implicit, chen2021minimax,  chen2022improved, chen2022policy, chen2021finding,  chen2022near,  jafarnia2021online, min2022learning}, for which we refer to \cite{tarbouriech2021stochastic, chen2021implicit} for a detailed comparison.

%% file: preliminaries.tex
\section{Preliminaries}\label{sec: prelim}
We consider a stochastic path (SP) problem with a finite state space $\calS$, a finite action set $\calA$, an initial state $s_\init\in\calS$, a terminal state $g$ (for notational simplicity, we let $g\notin\calS$), a transition kernel $P: \calS\times \calA\rightarrow \Delta_{\calS\cup \{g\}}$, and a reward function $r: \calS\times \calA\rightarrow [-1,1]$. We define $S=|\calS|$ and $A=|\calA|$. In an \emph{episode}, the player starts from the initial state $s_1= s_\init$\footnote{This assumption is for simplicity and is without loss of generality -- if the initial state is drawn from a fixed distribution $\rho\in\Delta_\calS$, we can create a virtual initial state $s_\init$ on which every action leads to zero reward and next state distribution~$\rho$. \label{fn: conversion}}. At the $i$-th step in an episode, the player sees the current state $s_i\in\calS$, takes an action $a_i\in\calA$, which leads to a reward value $r(s_i,a_i)$ and generates the next state $s_{i+1}$ according to $s_{i+1}\sim P_{s_i,a_i}(\cdot)$. The episode terminates right after the player reaches state $g$ (no action is taken on $g$). We assume that the $r$ is known to the learner, while $P$ is not. We call the problem \emph{stochastic longest path} (SLP) if $r(s,a)\geq 0$ for all $s,a$, and call it \emph{stochastic shortest path} (SSP) if $r(s,a)\leq 0$ for all $s,a$.

A history-dependent deterministic policy $\pi=(\pi_1, \pi_2, \ldots)$ is a mapping from state-action histories to actions, i.e., $\pi_i: (\calS\times \calA)^{i-1}\times \calS\rightarrow \calA$; we use $\piHD$ to denote the set of all history-dependent deterministic policies. A stationary deterministic policy $\pi$ is a mapping from states to actions, i.e., $\pi: \calS\rightarrow \calA$; we use $\piSD$ to denote the set of all stationary deterministic policies. The state value function of a policy $\pi\in\piHD$ on state $s\in\calS$ are defined as 
\begin{align*}
    V^\pi(s) &\triangleq \E^{\pi}\left[\sum_{i=1}^\tau r(s_i, a_i)~\Bigg|~ s_1=s\right],
\end{align*}
where $\tau=\min \{i:~s_{i+1}=g\}$, i.e., the timestep right before reaching the terminal state $g$ (or $\infty$ if $g$ is never reached), and $\E^{\pi}$ denotes expectation under policy $\pi$. Naturally, $V^{\pi}(g)\triangleq0$. 

A policy $\pi$ is called \emph{proper} if $g$ is reached with probability~$1$ under policy $\pi$ starting from any state. 
In this paper, we make the following assumption: 
\begin{assumption}\label{assum: proper}
    All policies in $\piHD$ are proper.  
\end{assumption}
\pref{assum: proper} is stronger than those in previous works on SSP \citep{rosenberg2020near, cohen2021minimax, tarbouriech2021stochastic, chen2021implicit}, which only require the \emph{existence} of a proper policy. We note that the algorithmic trick they developed (adding a small amount of cost to every step) can also help us weaken \pref{assum: proper}. 
More details on this are available in \pref{app:nonproper}. 

By Theorem~7.1.9 of \cite{puterman2014markov}, \pref{assum: proper} implies that there is a stationary and deterministic optimal policy $\pi^\star\in\piSD$ such that $V^{\pi^\star}(s)\geq V^{\pi}(s)$ for any $\pi\in\piHD$ and any $s$. We let $V^\star(\cdot)\triangleq V^{\pi^\star}(\cdot)$ and define 
\begin{align*}
    \Vinit \triangleq \left| V^\star(s_\init)\right|,  \quad
       \Bstar \triangleq \max_s \left|V^\star(s)\right|. 
\end{align*}
To establish our result, we also need the following definitions: 
\begin{definition}
Define
   \begin{align*}
       \Rinit &\triangleq \sup_{\pi\in\piHD}\sqrt{ \E^\pi\left[\left(\sum_{i=1}^{\tau} r(s_i,a_i)\right)^2~\Bigg|~s_1=s_\init\right]}, \\
       \Rmax &\triangleq  \max_s\sup_{\pi\in\piHD}\sqrt{ \E^{\pi}\left[\left(\sum_{i=1}^{\tau} r(s_i,a_i)\right)^2~\Bigg|~s_1=s\right]}, \\ 
       \Tmax &\triangleq \max_s\sup_{\pi\in\piHD} \E^\pi\big[\tau~|~s_1=s\big],  
    \end{align*}
    where $\tau=\min \{i:~s_{i+1}=g\}$, i.e., the timestep right before reaching the terminal state $g$. 
\end{definition} 
In words, $\Tmax$ is the maximum (over all policies and all states) expected time to reach the terminal state; $\Rinit$ and $\Rmax$ are two quantities that represent the \emph{range} of the total reward in an episode. Notice that in the definition of $R$, the starting state is fixed as $s_{\init}$, while in the definition of $\Rmax$, a maximum is taken over all possible starting states. 

Under \pref{assum: proper}, $\Vinit$, $\Bstar$, $R$, $\Rmax$, and $\Tmax$ are all bounded. For simplicity, we assume that they are all $\geq 1$. 

\subsection{Learning Protocol}
The learning procedure considered in this paper is the same as previous works on SSP. We let the learner interact with the SP for $K$ episodes, each started from $s_\init$. We define the regret as the difference between $KV^\star(s_\init)$ (the expected total reward obtained by the optimal policy) and the total reward of the learner. We keep a time index $t$ to track the number of steps executed by the learner, and let $s_t$ denote the state the learner sees at time $t$. Episode $k$ starts at time $t_k$, and thus $s_{t_k}=s_\init$. At time $t$, the learner takes an action $a_t$, and transitions to $s_t'\sim P_{s_t,a_t}(\cdot)$. If $s_t'\neq g$, we let $s_{t+1}=s_t'$; otherwise, we let $t_{k+1} = t+1$ to be the first step of episode $k+1$. The reader can refer to \pref{alg: general case} to see how the time indices are updated. The regret can be written as 
\begin{align*}
    \Reg_K = \sum_{k=1}^K \left(V^\star(s_\init) - \sum_{t=t_k}^{e_k}r(s_t, a_t)\right),    
\end{align*}
with $e_k = t_{k+1} - 1$. 
We let $T$ to be the total number of steps during $K$ episodes. That is, $T= e_{K}$. 

\subsection{Notation}
For $x>0$, define $\ln_+(x)\triangleq \ln(1+x)$. We write $x=O(y)$ or $x\leq O(y)$ to mean that $x\leq cy$ for some universal constant $c$, and write $x=\tilde{O}(y)$ or $x\leq \tilde{O}(y)$ if $x\leq cy$ for some $c$ that only contains logarithmic factors. $\E_t[\cdot]$ denotes expectation conditioned on history before time $t$. $[n]$ denotes the set $\{1, 2, ,\ldots, n\}$. We define $\tiliota_{T,B,\delta} \triangleq (\ln(SA/\delta) + \ln\ln(BT))\times \ln T$. $\mathbb{V}(P,V)$ where $P\in\Delta_S$ and $V\in\mathbb{R}^S$ denotes the variance of $V$ under $P$, i.e., $\mathbb{V}(P,V)\triangleq \sum_{i=1}^S P(i)V(i)^2 - (\sum_{i=1}^S P(i)V(i))^2$.

%% file: general.tex
\section{An Algorithm for General Stochastic Path (SP)} \label{sec: general SP}

Our algorithm for general SP is \pref{alg: general case}, which is simplified from the SVI-SSP algorithm by \cite{chen2021implicit}. The inputs are a parameter $B$ that is supposed to an upper bound of $\Bstar$, and a confidence parameter $\delta$. The algorithm maintains an optimistic estimator $Q(s,a)$ of $Q^\star(s,a):=r(s,a) + \E_{s'\sim P_{s,a}}[V^\star(s')]$ (i.e., with high probability, $Q(s,a)\geq Q^\star(s,a)$ always holds). In every step $t$, the learner chooses action $a_t=\argmax_a Q(s_t,a)$ based on the ``optimism in the face uncertainty'' principle (\pref{line: greedy choose a_t}), and updates the entry $Q(s_t,a_t)$ after receiving the reward and the next state, with an additional exploration bonus $b_t$ that keeps the optimism of $Q(s_t,a_t)$ (\pref{line: calculate bar P}--\pref{line: update Q_t}). Although this algorithm is similar to the one in \cite{chen2021implicit}, the existing analysis only applies to SSP and SLP, and it is unclear how it handles general SP. Our main contribution in this section is to provide a regret guarantee for this algorithm in general SP. 

The regret guarantee of  \pref{alg: general case} is given by the following theorem. 

\setcounter{AlgoLine}{0}
\begin{algorithm}[t]
    \caption{VI-SP}\label{alg: general case}
\nl    \textbf{input}: $B\geq 1$, $0<\delta<1$, sufficiently large universal constants $c_1, c_2$ that satisfy $2c_1^2\leq c_2$. \\
\nl    \textbf{Initialize}: $t\leftarrow 0$, $s_1\leftarrow s_\init$. \\
\nl    For all $(s,a,s')$ where $s\neq g$, set
    \begin{align*}
        n(s,a,s')=n(s,a)\leftarrow 0, \quad Q(s,a)\leftarrow  B, \quad V(s)\leftarrow B. 
    \end{align*} 
\nl    Set $V(g)\leftarrow 0$. \\
\nl    \For{$k=1,\ldots, K$}{
\nl        \While{\text{true}}{
\nl        $t \leftarrow t+1$ \\
\nl        \algComment{$Q_t(s,a), V_t(s)$ are defined as the $Q(s,a), V(s)$ at this point.} \\[0.1cm]
\nl        Take action $a_t=\argmax_a Q(s_t,a)$, receive reward $r(s_t,a_t)$, and transit to $s_{t}'$. \label{line: greedy choose a_t}\\
\nl        Update counters: $n_t\triangleq n(s_t,a_t)\leftarrow n(s_t,a_t)+1$, $n(s_t,a_t,s_t')\leftarrow n(s_t,a_t,s_t')+1$. \\
\nl        Define $\bar{P}_{t}(s')\triangleq   \frac{n(s_t,a_t,s')}{n_t}\ \forall s'$. \label{line: calculate bar P}\\
\nl        Define $b_t\triangleq  \max\Big\{c_1\sqrt{\frac{\mathbb{V}(\bar{P}_t, V)\iota_t}{n_t}}, \frac{c_2B\iota_t}{n_t}\Big\}$, where $\iota_t = \ln(SA/\delta) + \ln\ln(Bn_t)$.  \\
\nl        $Q(s_t,a_t)\leftarrow 
        \min\left\{ r(s_t,a_t) + \bar{P}_tV + b_t, Q(s_t,a_t) \right\}$ \label{line: update Q_t}\\
\nl        $V(s_t)\leftarrow \max_a Q(s_t,a)$. 
        
\nl        \lIf{$s_t'\neq g$}{ then $s_{t+1}\leftarrow s_t'$} 
\nl        \lElse{ $s_{t+1}\leftarrow s_\init$ and \textbf{break}}
        }
        
    }
\end{algorithm}

\begin{theorem}\label{thm: main for mixed case}
    If \pref{assum: proper} holds, then \pref{alg: general case} with $B\geq \Bstar$ ensures that with probability at least $1-O(\delta)$, for all $K\geq 1$, with $T$ being the total number of steps in $K$ episodes, 
    \begin{align*}
        \Reg_K = O\left(R\sqrt{SAK\tiliota_{T,B,\delta}} + \Rmax SA\ln\left(\frac{\Rmax K}{R\delta}\right)\tiliota_{T,B,\delta} + BS^2A\tiliota_{T,B,\delta}\right), 
    \end{align*}
    where $\tiliota_{T,B,\delta} \triangleq (\ln(SA/\delta) + \ln\ln(BT))\times \ln T$.\footnote{Technically, the total number of steps $T$ is a random quantity but can be replaced by $KT_{\max}$ with high probability if desired.}
\end{theorem}
The proof of \pref{thm: main for mixed case} can be found in \pref{app: general upper}.  \pref{thm: main for mixed case} generalizes previous works on near-optimal near-horizon-free regret bounds for RL \citep{zhang2021reinforcement, tarbouriech2021stochastic, chen2021implicit}. Specifically, with a closer look into their analysis, one can find that their analysis 
leads to a regret bound that depends on the magnitude of $\sum_{i\in\text{episode}} |r(s_i,a_i)|$, which can be much larger than $|\sum_{i\in\text{episode}} r(s_i,a_i)|$ if the rewards have mixed signs. To address this issue, we develop new analysis techniques to get a near-horizon-free regret bound, which only scales with $|\sum_{i\in\text{episode}} r(s_i,a_i)|$. 
Other than this, the rest of the proofs are similar to those in \cite{chen2021implicit} with simplifications. Unlike prior work, our analysis does not involve the intricate ``high-order expansion'' as seen in \cite{zhang2021reinforcement, tarbouriech2021stochastic, chen2021implicit}, the possibility of which is hinted by \cite{zhang2022horizon}. 

\begin{proof}\textbf{sketch for \pref{thm: main for mixed case}~~}
   We first connect the regret with the sum of \emph{advantages}, $\sum_{t=1}^T (V^\star(s_t) - Q^\star(s_t,a_t))$. This is standard based on the performance difference lemma \citep{kakade2002approximately}. 
   
   We use the fact that $B\geq \Bstar$ and the bonus construction to show that the value estimator $Q(s,a)$ always upper bounds $Q^\star(s,a)$ with high probability (\pref{lem: MVP}). This relies on the monotonic value propagation idea developed by \cite{zhang2021reinforcement}. Then following the analysis of \cite{zhang2021reinforcement} and \cite{chen2021implicit}, we can show the following high probability bound (\pref{lem: V*-Q*}): 
   \begin{align}
       \sum_{t=1}^T (V^\star(s_t) - Q^\star(s_t,a_t)) \leq \tilde{O}\left(\sqrt{SA\sum_{t=1}^T \mathbb{V}(P_{s_t,a_t}, V^\star) } + BS^2A\right). \label{eq: V*-Q* bound sketch} 
   \end{align}
   The way we bound $\sum_{t=1}^T \mathbb{V}(P_{s_t,a_t}, V^\star)$ is the key to handle the case where the rewards have mixed signs. Specifically, we show the following (\pref{lem: getting a high prob bound}): 
   \begin{align}
        \sum_{t=1}^T \mathbb{V}(P_{s_t,a_t}, V^\star) &\leq \tilde{O}\left(\Rmax \sum_{t=1}^T (V^\star(s_t) - Q^\star(s_t,a_t))  + R^2 K + \Rmax^2\right). \label{eq: sum var V* bound sketch}
    \end{align}
    Combining \pref{eq: V*-Q* bound sketch} and \pref{eq: sum var V* bound sketch} and solving for $\sum_{t=1}^T (V^\star(s_t) - Q^\star(s_t,a_t))$, we get an upper bound for it, which in turn gives a high-probability regret bound. 
\end{proof}

\subsection{Lower bound}


In this subsection, we show that the upper bound established in \pref{thm: main for mixed case} is nearly tight. The proofs for this subsection can be found in \pref{app: general lower}. 
\begin{theorem}\label{thm: lower bound for general case}
   For any $u\geq 2$, and $K\geq \Omega(SA)$, we can construct a set of SP instances such that $\Rinit\leq u$ for all instances, and there exists a distribution over these instances such that the expected regret of any algorithm is at least $\Omega(u\sqrt{SAK})$. 
\end{theorem}
Specially, in the lower bound construction of \pref{thm: lower bound for general case}, $\Vinit$ and $\Bstar$ are of order $O(1)$ for all $u\leq \sqrt{K/(SA)}$, showing that $\Vinit$ and $\Bstar$ are insufficient to characterize the regret bound for general SP problems. As we will see in the following sections, this contrasts with the special cases SLP and SSP, where except for logarithmic terms, the coefficients in the regret bound can be completely characterized by $\Vinit$ and $\Bstar$. 

The quantity $R$ in the regret upper and lower bounds (\pref{thm: main for mixed case} and \pref{thm: lower bound for general case}) is undesirable because its definition involves a \emph{supremum} over all policies, which might be very large. Is it possible to refine the upper bound so that it only depends on quantities that correspond to the \emph{optimal policy}? Specifically, we define 
\begin{align*}
       \Rstar &~\triangleq~ \max_s \sqrt{ \E^{\pi^\star}\left[\left(\sum_{i=1}^{\tau} r(s_i,a_i)\right)^2~\Bigg|~s_1=s\right]}, 
\end{align*} 
and ask: can the regret bound only depends on $\Rstar$? Notice that \pref{thm: lower bound for general case} is uninformative for this question because $R\approx \Rstar$ in its construction. The next theorem gives a negative answer \emph{when the learner is agnostic of the value of $\Rstar$}. 
\begin{theorem}\label{thm: impossibility of R*}
     Let $u\geq 2$ be arbitrarily chosen, and let $K\geq \Omega(SA)$. For any algorithm that obtains a expected regret bound of $\tilde{O}(u\sqrt{SAK})$ for all problem instances with $\Rstar=\Rmax\leq u$, there exists a problem instance with $\Rstar=O(1)$ and $\Rmax\leq u$ but the expected regret is at least  $\tilde{\Omega}(u\sqrt{SAK})$. 
\end{theorem}
Given \pref{thm: impossibility of R*}, a left open question is whether $\tilde{O}(\Rstar\sqrt{SAK})$ is achievable when the learner has information about $\Rstar$. Note that our algorithm \pref{alg: general case} only requires knowledge of $\Bstar$, which, in general, does not provide information about $\Rstar$ (e.g., in the construction of \pref{thm: impossibility of R*}, $\Bstar=O(1)$ for all $u\leq \sqrt{K/(SA)}$). Therefore, an algorithm with such a refined guarantee would be quite different from our algorithm.

%% file: reward.tex
\section{Stochastic Longest Path (SLP)}
For the special case SLP where $r(\cdot,\cdot)\geq 0$, we first demonstrate that our general result in \pref{thm: main for mixed case} already gives a nearly tight bound. The following lemma connects the notion of $R, \Rmax$ in general SP to $\Vinit, \Bstar$ in SLP. 
\begin{lemma}\label{lem: connecting reward and general}
    If $r(s,a)\geq 0$ for all $s,a$, then $R = O(\sqrt{\Vinit \Bstar\ln_+(\Bstar/\Vinit)})$ and $\Rmax = O(\Bstar)$, where $\ln_+(x)\triangleq \ln(1+x)$.  
\end{lemma}
\pref{lem: connecting reward and general} together with \pref{thm: main for mixed case} immediately implies the regret guarantee for SLP. Specifically, assuming that $B\geq \Bstar$, combining \pref{thm: main for mixed case} and \pref{lem: connecting reward and general} yields
\begin{align}
    \Reg_K = O\left(\sqrt{\Vinit \Bstar SAK\ln_+(\Bstar/\Vinit)\tiliota_{T,B,\delta}} + \Bstar SA \ln\left(\frac{\Bstar K}{\Vinit \delta}\right)\tiliota_{T,B,\delta} + BS^2A \tiliota_{T,B,\delta}\right). \label{eq: reg SLP bad} 
\end{align}
The logarithmic terms in \pref{eq: reg SLP bad} can be slightly improved if we follow a different approach to bound the sum of variance $\sum_{t}\mathbb{V}(P_{s_t,a_t}, V^\star)$. That is, using \pref{lem: cost sm of var} instead of using \pref{lem: getting a high prob bound}. Note that the proof of \pref{lem: cost sm of var} is similar to those of previous works \citep{zhang2021reinforcement, tarbouriech2021stochastic, chen2021implicit}, which leads to a regret bound that depends on the magnitude of $\sum_{i\in\text{episode}}|r(s_i,a_i)|$ instead of $|\sum_{i\in\text{episode}} r(s_i,a_i)|$. Therefore, while it does not work for general SP, we can use it for SLP. Comparing \pref{eq: reg SLP bad} and the bound in \pref{thm: main thm for reward}, we see that specializing our general result in \pref{thm: main for mixed case} to SLP only leads to looseness in logarithmic factors. 
\begin{theorem}\label{thm: main thm for reward}
    If \pref{assum: proper} holds and $r(\cdot,\cdot)\geq 0$, then \pref{alg: general case} with $B\geq \Bstar$ ensures that with probability at least $1-\delta$, for all $K\geq 1$, with $T$ being the total number of steps in $K$ episodes,
    \begin{align}
        \Reg_K = O\left(\sqrt{\Vinit \Bstar SAK\tiliota_{T,B,\delta}} + BS^2A \tiliota_{T,B,\delta}\right).   \label{eq: regret of SLP}
    \end{align}
\end{theorem}
The proof of \pref{thm: main thm for reward} is in \pref{app: upper SLP}.

\subsection{Algorithm without knowledge of $\Bstar$}\label{sec: algorithm without B* knowledge}
While \pref{thm: main thm for reward} gives a near-optimal bound, it is unclear how to make the algorithm work if prior knowledge on $\Bstar$, which we need in order to set the value of $B$, is unavailable. Here, we first present a passive way to set $B$ that is simple but leads to a highly sub-optimal bound. Observe from \pref{thm: main thm for reward} that $B$ only appears in the ``lower-order'' term in the regret bound. Therefore, a simple idea is to set $B$ to be something large (of order $\sqrt{K/S^3A}$) with the hope that $B\geq \Bstar$ will hold in a wide range of cases. With this choice, if $\Bstar\leq B$ indeed holds, then we enjoy a regret bound of $\tilde{O}(\sqrt{\Vinit\Bstar SAK} + BS^2A)=\tilde{O}(\sqrt{\Vinit\Bstar SAK} + \sqrt{SAK})=\tilde{O}(\sqrt{\Vinit\Bstar SAK})$; if $\Bstar>B$, then we simply bound the regret by $\Vinit K\leq O(\Vinit\Bstar^2S^3A)$, where the last inequality is implied by $\Bstar\geq B=\Theta(\sqrt{K/(S^3A)})$. Overall, this simple approach gives a regret bound of 
\begin{align}
    \tilde{O}\left(\sqrt{\Vinit \Bstar SAK} + \Vinit \Bstar^2 S^3A\right).  \label{eq: bad bound for agnostic reward} 
\end{align}
While the dominant is optimal, the lower-order term has cubic dependency  (i.e., $\Vinit \Bstar^2$) on the scale of the cumulative reward, which is unnatural, and can easily overwhelm the dominant term when $\Vinit \Bstar^2S^3A \gtrsim \sqrt{\Vinit\Bstar SAK}$, or $\Bstar\geq (K/(\Vinit S^5A))^{1/3}$. Previous prior-knowledge-free algorithms for SSP \citep{cohen2021minimax, tarbouriech2021stochastic, chen2021implicit} also suffer from this issue and have at least cubic dependency on the scale of cumulative reward. 

\begin{algorithm}[t]
    \caption{Procedure to estimate $\Vinit$ in SLP} \label{alg: estimating V*}
    \textbf{input}: $\zeta\geq 1$, $U> 1$. \\
    \For{$i=1,\ldots, \lceil \log_2  U\rceil$}{
        Initiate a \pref{alg: general case} with $B=2^i\zeta$ and probability parameter as $\delta'=\delta/\lceil \log_2 U\rceil$ (call this instance \alg). \\
        Run \alg until $N\geq 16c^2\zeta S^2A \tiliota_{M,B,\delta'}$, where $N$ is the number of episodes, $M$ is the total number of steps, and $c$ is the universal constant hidden in the $O(\cdot)$ notation in \pref{eq: regret of SLP}.  \\
        Let $\hat{r}_i$ be average reward of \alg in these $N$ episodes (i.e., the total reward divided by $N$). 
    }
    \textbf{return} $\hat{V}\triangleq 2\max_i\{\hat{r}_i\}$.
\end{algorithm}
\begin{algorithm}[t]
    \caption{VI-SLP for unknown $\Bstar$} \label{alg: B* agnostic}
    \textbf{input}: $\zeta\geq 1$, $U> 1$.\\ 
    Run \pref{alg: estimating V*} with inputs $\zeta$ and $U$, and get output $\hat{V}$. \\
    Run \pref{alg: general case} with input $B=\hat{V}\zeta$ in the rest of the episodes. 
\end{algorithm}

In this subsection, we introduce a way to obtain a regret guarantee that only (nearly) linearly depends on the scale. Observe that in SLP, $\Bstar$ corresponds to the maximum total expected reward the learner can get starting from any state. Clearly, the learner needs to have some knowledge about the \emph{optimal policy} in order to estimate this quantity. Fortunately, it needs not to be accurately estimated; an estimation up to a constant factor suffices. Therefore, a reasonable plan is to coarsely estimate $\Bstar$ up to a constant factor, and then use the estimation to set $B$. Notice that estimating $\Bstar$ requires estimating $V^\star(s)$ for all $s$ since $\Bstar=\max_s V^\star(s)$. 

While we can indeed make this idea work (details omitted), we find that an even more economical solution is to just estimate $\Vinit=V^\star(s_\init)$ and set $B$ to be something large compared to this estimation. 
Below we explain this idea. Let's first assume that $\frac{\Bstar}{\Vinit}\leq \zeta$ for some fixed $\zeta$ (will be relaxed later). Now consider running \pref{alg: general case} for $N=\tilde{\Theta}(\zeta S^2A)$ episodes with parameter $B=V\zeta$ for some value $V$ that we choose. If we happen to choose a $V\in[\Vinit, 2\Vinit]$, then we have $B=V\zeta \geq \Vinit\zeta \geq \Bstar$, and thus the regret bound in \pref{thm: main thm for reward} holds. Let $\hat{r}$ be the average reward in these $N$ episodes (i.e., total reward divided by $N$). Then \pref{thm: main thm for reward} gives 
\begin{align}
    \Vinit - \hat{r} \leq \tilde{O}\left(\sqrt{\frac{\Vinit \Bstar SA}{N}} + \frac{BS^2A}{N}\right) = O\left( \sqrt{\frac{\Vinit\Bstar }{\zeta S}} + \frac{B}{\zeta } \right) = O\left(\frac{\Vinit}{\sqrt{S}}+ V\right) = O(\Vinit), \label{eq: V is close}
\end{align}
where in the first equality we use $N=\tilde{\Theta}(\zeta S^2A)$, in the second equality we use $\frac{\Bstar}{\Vinit}\leq \zeta$ and $B=V\zeta$, in the third equality we use $V\leq 2\Vinit$. By setting $N$ to be large enough, we can ensure that the $O(\Vinit)$ on the right-hand side is no more than $\frac{1}{2}\Vinit$, which then gives $\frac{1}{2}\Vinit\leq \hat{r}\leq \frac{3}{2}\Vinit$. 

On the other hand, if we choose some $V$ that is not in the range of $[\Vinit, 2\Vinit]$ and set $B=V\zeta$, we may not have a good guarantee like in \pref{eq: V is close}. However, the following reversed inequality must hold no matter how large $V$ is:  
\begin{align}
    \Vinit - \hat{r} \geq -\tilde{O}\left(\sqrt{\frac{\Vinit\Bstar}{N}} + \frac{\Bstar}{N}\right) = -\tilde{O}\left(\sqrt{\frac{\Vinit\Bstar}{\zeta S^2A}} + \frac{\Bstar}{\zeta S^2A}\right) =-\tilde{O}(\Vinit),  \label{eq: reverse Freedman}
\end{align}
where the first inequality is by the fact that $\Vinit \geq \E[\hat{r}]$ (because $\Vinit$ is the expected value of the \emph{optimal} policy) and that we can use Freedman's inequality to lower bound $\Vinit - \hat{r}$ (details given in the formal proof). This inequality gives $\hat{r}\leq O(\Vinit)$ no matter what $V$ we use. 

With the observations from \pref{eq: V is close} and \pref{eq: reverse Freedman}, we have the following strategy to estimate $\Vinit$ \emph{given that $\frac{\Bstar}{\Vinit}\leq \zeta$ holds}: 
we perform the procedure described above for every $V\in \{1, 2, 4, 8, \ldots\}$. Let $\hat{r}_i$ denote the average reward when we use $V=2^{i}$. By the argument in \pref{eq: V is close}, at least one of the $\hat{r}_i$'s is of order $\Theta(\Vinit)$; by the argument in \pref{eq: reverse Freedman}, all $\hat{r}_i$'s are of order $O(\Vinit)$. Combining them, we have $\Vinit=\Theta(\max_i\{\hat{r}_i\})$. This procedure to estimate $\Vinit$ is formalized in \pref{alg: estimating V*}, with its guarantee
given in the following lemma: 
\begin{lemma}\label{lem: for the estimating procedure}
    Suppose that $\frac{\Bstar}{\Vinit}\leq \zeta$, and $\Vinit\leq U$. Then \pref{alg: estimating V*} with inputs $\zeta$ and $U$ ensures that its output $\hat{V}$ satisfies $\Vinit \leq \hat{V}\leq 3\Vinit$. 
\end{lemma}
With $\hat{V}$ from \pref{alg: estimating V*} being a coarse estimation of $\Vinit$, we can use it to set the parameter $B$ in \pref{alg: general case} as $B=\hat{V}\zeta$. Our overall algorithm is presented in \pref{alg: B* agnostic}. However, notice that 
\pref{lem: for the estimating procedure} only gives a meaningful guarantee when the two conditions $\frac{\Bstar}{\Vinit}\leq \zeta$ and $\Vinit\leq U$ hold. Apparently, they do not hold for all instances. In the theorem below, we show that with appropriate choices of $\zeta$ and $U$, the additional regret due to their violation is well-bounded. 
\begin{theorem}\label{thm: main for parameter free reward setting}
In the case when the learner has access to an absolute upper bound for $\Vinit$ (for example, $\Tmax$ is an absolute upper bound for $\Vinit$), by setting $U$ to be that upper bound and setting $\zeta=\sqrt{K/(S^2A\ln U)}$, \pref{alg: B* agnostic} ensures with probability at least $1-O(\delta)$
\begin{align*}
    \Reg = O\left(\Bstar\sqrt{S^2AK\ln U}\tiliota_{T,KU,\delta}\right).
\end{align*}
Alternatively, with $\zeta=\sqrt{K/(S^3A\ln U)}$, \pref{alg: B* agnostic} ensures with probability at least $1-O(\delta)$
\begin{align*}
    \Reg = O\left( \sqrt{\Vinit \Bstar SAK \ln U}\tiliota_{T,KU, \delta} + \frac{\Bstar^2S^3A \ln U}{\Vinit}\tiliota_{T,KU,\delta}\right).  
\end{align*}
In the case when the learner has NO access to an absolute upper bound for $\Vinit$, we set $U=K^{1/\epsilon}$ for some parameters $\epsilon \in(0,1)$. With $\zeta=\sqrt{K/(S^2A\ln U)}$ and $\zeta=\sqrt{K/(S^3A\ln U)}$, \pref{alg: B* agnostic} ensures with probability at least $1-O(\delta)$
\begin{gather*}
    \Reg=O\left(\Bstar\sqrt{\epsilon^{-1}S^2AK\ln K}\tiliota_{T,K,\delta\epsilon} + \Vinit^{1+\epsilon}\right) \\
\text{and}\qquad \Reg=O\left(\sqrt{ \epsilon^{-1}\Vinit \Bstar SAK \ln K}\tiliota_{T,K, \delta\epsilon} + \frac{\epsilon^{-1}\Bstar^2S^3A \ln K}{\Vinit}\tiliota_{T,K,\delta\epsilon}+ \Vinit^{1+\epsilon}\right), 
\end{gather*}
respectively. 

\end{theorem}
The proof can be found in \pref{app: upper SLP}. In \pref{thm: main for parameter free reward setting}, we obtain two regret bounds for SLP without knowledge of $\Bstar$: $\tilde{O}(\Bstar \sqrt{S^2 AK})$ and $\tilde{O}(\sqrt{\Vinit\Bstar SAK} + \frac{\Bstar^2}{\Vinit}S^3A)$, both not matching the bound $O(\sqrt{\Vinit\Bstar SAK} + \Bstar S^2A)$ in \pref{thm: main thm for reward} for the case with a known $\Bstar$. Is it possible to close the gap between the ``known $\Bstar$'' and the ``unknown $\Bstar$'' cases? In \pref{sec: lower bound for agnostic reward setting}, we will show that this is \emph{impossible}, by giving a regret lower bound for algorithms agnostic of $\Bstar$. The lower bound can be strictly larger than the upper bound with knowledge of $\Bstar$, thus formally identifying the \emph{price of information about $\Bstar$} for SLP.  Finally, we remark on the source of suboptimality in the bounds in \pref{thm: main for parameter free reward setting}. The bounds we get are $\Bstar \sqrt{S^2AK}$ and $\sqrt{\Vinit \Bstar SAK} + \frac{\Bstar^2}{\Vinit}S^3A$. The additional $S$ dependencies come from the $S^2$ in the lower-order term in \pref{thm: main thm for reward}. It is conjectured by previous work \citep{zhang2021reinforcement} that this $S^2$ in the lower-order term can be improved to $S$. If this conjecture is true, then our bounds in \pref{thm: main for parameter free reward setting} can be improved to $\Bstar\sqrt{SAK}$ and $\sqrt{\Vinit \Bstar SAK} + \frac{\Bstar^2}{\Vinit}SA$. As we will see in \pref{sec: lower bound for agnostic reward setting}, these bounds are unimprovable when $\Bstar$ is unknown.

\subsection{Lower bound for algorithms agnostic of $\Bstar$}\label{sec: lower bound for agnostic reward setting}
If the magnitude of $\Bstar$ is known, \pref{thm: main thm for reward} shows that a regret bound of $\tilde{O}(\sqrt{\Vinit\Bstar SAK} + \Bstar S^2A)$ is possible. In this section, we show that there is a price to pay for adaptivity.

\begin{theorem}\label{thm: reward setting lower bound for agnostic alg}
In SLP, for any algorithm agnostic to $\Bstar$ that obtains a regret bound of
$\tilde{O}(\nu\sqrt{SAK})$ for any problem instance where $\Bstar=\Vinit=\nu$ and sufficiently large $K$, there exists a problem instance with $\Vinit\leq 1 + 2\nu$, $\Bstar = \tilde{O}\big(\nu\sqrt{K/(SA)}\big)$ such that the regret is at least $\Omega(\nu K)$.
\end{theorem}
See \pref{app: lower bound SLP} for the proof. This theorem implies that being agnostic to $\Bstar$ is fundamentally harder than knowing an order optimal bound on $\Bstar$, since \pref{thm: main thm for reward} obtains sub-linear regret of order $\tilde{O}(\nu(SA)^\frac{1}{4}K^\frac{3}{4})$ in the hard instance mentioned in \pref{thm: reward setting lower bound for agnostic alg}.
Considering two classes of upper bounds, one in which we always scale with $\sqrt{K}$ without dominating lower order terms, and one in which we allow a constant cost for adapting to an unknown $\Bstar$, we directly derive the following results.
\begin{corollary}\label{corr: reward setting lower bound for agnostic alg}
Any algorithm with an asymptotic upper bound of
\begin{align*}
    \tilde{O}\left(\Bstar^\alpha\Vinit^{1-\alpha}\sqrt{SAK}\right)+o\left(\Bstar^2\right)\,,
\end{align*} 
satisfies at least $\alpha \geq 1$ and
any algorithm with an upper bound of 
\begin{align*}
    O\left(\sqrt{\Vinit\Bstar SAK}+ \left(\frac{\Bstar}{\Vinit}\right)^2\poly(\Vinit, S, A)\right)
\end{align*}
requires the constant term to be at least $\tilde{\Omega}\left(\frac{\Bstar^2SA}{\Vinit}\right)$.  
\end{corollary}
\begin{proof}
For the first part, note that for any $\alpha<1$, the regret bound for the bad case in \pref{thm: reward setting lower bound for agnostic alg} with $\nu=O(1)$ reads $\Bstar^\alpha\Vinit^{1-\alpha}\sqrt{SAK}+o(\Bstar^2)=O(K^{\frac{1}{2}(1+\alpha)}(SA)^{\frac{1}{2}(1-\alpha)})+o(K)$, which is sublinear in $K$ and hence constitutes a contradiction.
Similarly, in the second case, we can make $K$ large enough such that the constant term (i.e., $(\frac{\Bstar}{\Vinit})^2\poly(\Vinit, S, A)$) is absorbed by the dominant term (i.e., $\sqrt{\Vinit \Bstar SAK}$) in the $\Vinit=\Bstar=\nu$ environment, which means we can apply \pref{thm: reward setting lower bound for agnostic alg} and the $\poly(\Vinit, S, A)$ term must be of order $\Omega(\Vinit SA)$, to satisfy the $\Omega(\nu K)$ lower bound.
\end{proof}

%% file: cost.tex
\section{Stochastic Shortest Path (SSP)}\label{sec: SSP}
SSP has been studied extensively recently. The works by \cite{tarbouriech2021stochastic} and \cite{chen2021implicit} have achieved a near-optimal regret bound $\tilde{O}(\sqrt{\Vinit \Bstar SAK} + \Bstar S^2A)$ when the knowledge on $\Bstar$ is available to the learner.\footnote{The upper bound reported in \cite{tarbouriech2021stochastic} and \cite{chen2021implicit} is of order $\Bstar\sqrt{SAK} + \Bstar S^2A$, which is larger than what we report here. This is simply because 
in their analysis, they upper bound $\Vinit$ by $\Bstar$. We redo their analysis and report their refined dependence on $\Vinit$ here. Similarly, the lower bound obtained in \cite{rosenberg2020near} is $\Bstar\sqrt{SAK}$ because in their lower bound construction, they only consider instances where $\Vinit$ and $\Bstar$ are of the same order. Hence, the upper bound we obtain here does not violate their lower bound. } When such knowledge is unavailable, they design a way to adjust $B$ on the fly, and achieve a regret bound of $\tilde{O}(\sqrt{\Vinit \Bstar SAK} + \Bstar^3 S^3A)$. 

In this section, we improve their results, showing that for SSP, a bound of $\tilde{O}(\sqrt{\Vinit \Bstar SAK} + \Bstar S^2A)$ is possible even without prior knowledge on $\Bstar$. This is a contrast with \pref{thm: reward setting lower bound for agnostic alg} and \pref{corr: reward setting lower bound for agnostic alg}, which show that for SLP, without prior knowledge on $\Bstar$, this bound is unachievable. 

Our algorithm is \pref{alg: cost case parameter free}. It is almost identical to \pref{alg: general case} with three main differences. First, $B$ is no longer an input parameter in \pref{alg: cost case parameter free}, but is an internal variable updated on the fly. Second, the initial $Q(s,a)$, $V(s)$ values are initialized as $0$ in \pref{alg: cost case parameter free}, instead of $B$, which is natural since $Q^\star(s,a), V^\star(s)\leq 0$ for SSP. Third, in \pref{line:start of while}--\pref{line:end of while}, the algorithm tries to find a large enough $B$ to set the bonus term $b_t$, so that the resulted $|Q(s_t,a_t)|$ is upper bounded by $B$. 

The operation in \pref{line:start of while}--\pref{line:end of while} and the corresponding analysis are the keys to our improvement. Recall that we hope to always have $0\geq Q_t(s,a)\geq Q^\star(s,a)$ to ensure optimism. Also, we want $B$ to be not too much larger than $\Bstar$ to avoid regret overhead. Let's assume $0\geq Q_t(s,a)\geq Q^\star(s,a)$ for all $s,a$ at time $t$. In \pref{line: attempt Qtmp}, we attempt to calculate $Q_{t+1}(s_t,a_t)$ (denoted as $Q_{t+1}^\tmp(s_t,a_t)$ below). If $B\geq \Bstar$ holds, then since $r(s_t,a_t) + \bar{P}_t V_t + b_t \geq Q^\star(s_t,a_t)$ by the same argument as in the proof of \pref{lem: MVP} and $Q_t(s_t,a_t)\geq Q^\star(s_t,a_t)$ by assumption, we have $0\geq Q^\tmp_{t+1}(s_t,a_t)\geq Q^\star(s_t,a_t)\geq -\Bstar\geq -B$ by the definition of $Q_{t+1}^\tmp(s_t,a_t)$ in \pref{line: attempt Qtmp}. Thus, $B$ will not be increased in \pref{line:end of while}. In short, \emph{if optimism always holds} (i.e., $Q_t(s,a)\geq Q^\star(s,a)$), we will only increase $B$ in \pref{line:end of while} when $B < \Bstar$, and thus $B<2\Bstar$ all the time. 

\setcounter{AlgoLine}{0}
\begin{algorithm}[t]
    \caption{VI-SSP for unknown $\Bstar$}\label{alg: cost case parameter free}
\nl    \textbf{input}: $0<\delta<1$, sufficiently large universal constants $c_1, c_2$ that satisfy $2c_1^2 \leq c_2$.  \\
\nl    \textbf{Initialize}: $B\leftarrow 1$, $t\leftarrow 0$, $s_1\leftarrow s_\init$. \\
\nl    For all $(s,a,s')$ where $s\neq g$, set
    \begin{align*}
        n(s,a,s')=n(s,a)\leftarrow 0, \quad Q(s,a)\leftarrow  0, \quad V(s)\leftarrow 0. 
    \end{align*} 
\nl    Set $V(g)\leftarrow 0$. \\
\nl    \For{$k=1,\ldots, K$}{
\nl        \While{\text{true}}{
\nl        $t \leftarrow t+1$ \\
\nl        \algComment{$Q_t(s,a), V_t(s), B_t$ are defined as the $Q(s,a), V(s), B$ at this point.} \\[0.1cm]
\nl        Take action $a_t=\argmax_a Q(s_t,a)$, receive reward $r(s_t,a_t)$, and transit to $s_{t}'$. \\
\nl        Update counters: $n_t\triangleq n(s_t,a_t)\leftarrow n(s_t,a_t)+1$, $n(s_t,a_t,s_t')\leftarrow n(s_t,a_t,s_t')+1$. \\
\nl        Define $\bar{P}_{t}(s')\triangleq   \frac{n(s_t,a_t,s')}{n_t}\ \forall s'$. \\
\nl        \While{true}{ \label{line:start of while}
\nl        Define $b_t\triangleq  \max\Big\{c_1\sqrt{\frac{\mathbb{V}(\bar{P}_t, V)\iota_t}{n_t}}, \frac{c_2B\iota_t}{n_t}\Big\}$, where $\iota_t = \ln(SA/\delta) + \ln\ln(Bn_t)$.  \\
\nl        $Q^\tmp(s_t,a_t)\leftarrow 
        \min\left\{ r(s_t,a_t) + \bar{P}_tV + b_t, Q(s_t,a_t) \right\}$  \label{line: attempt Qtmp}  \\
\nl        \lIf{$|Q^\tmp(s_t,a_t)|\leq B$}{
            $Q(s_t,a_t)\leftarrow Q^\tmp(s_t,a_t)$ and \textbf{break}
        }
\nl   $B\leftarrow 2B$.  \label{line:end of while} 
        }
\nl     \algComment{$b_t$ and $\iota_t$ are defined as the $b_t$ and $\iota_t$ at this point.}  \label{line: bt and iotat position} \\
\nl        $V(s_t)\leftarrow \max_a Q(s_t,a)$.\\
\nl        \lIf{$s_t'\neq g$}{ then $s_{t+1}\leftarrow s_t'$} 
\nl        \lElse{ $s_{t+1}\leftarrow s_\init$ and \textbf{break}}
        }
        
    }
\end{algorithm}
The question then is how to show that optimism holds along the way. Because we start from $B=1$ and only increases $B$ when we are sure that $B<\Bstar$, one might suspect that the bonus term $b_t$ defined through $B$ is insufficient at the beginning, and the optimism might fail. Because of this, \cite{tarbouriech2021stochastic} and \cite{chen2021implicit} bound the regret in the $B<\Bstar$ regime by a term linear in $K$. However, one key observation in the analysis is that the original purpose of the bonus term is to compensate the deviation of $(P_{s_t,a_t}-\bar{P}_t)V_t$, where $|V_t|\leq B$ by our algorithm. Since $V_t$ is history-dependent, a common trick in the analysis \citep{azar2017minimax, zhang2021reinforcement} is to replace $V_t$ by $V^\star$ and bound the deviation of $(P_{s_t,a_t}-\bar{P}_t)V^\star$ using Freedman's inequality, for which a bonus term defined through $\Bstar$ is required. To deal with our case, instead of replacing $V_t$ by $V^\star$, we replace it by $V_{[B]}^{\star}\triangleq \max\{-B, V^\star\}$ and use Freedman's inequality on $(P_{s_t,a_t}-\bar{P}_t)V_{[B]}^{\star}$, for which a bonus term defined through $B$ suffices. We can further connect $V^\star_{[B]}$ back to $V^\star$ using the property $V^\star_{[B]}\geq V^\star$. The details are provided in \pref{lem: MVP for cost}, where we show that with high probability, $Q_{t}(s,a)\geq Q^\star(s,a)$ holds for all $t,s,a$, even if $B$ is smaller than $\Bstar$ along the learning process. The formal guarantee of our algorithm is given by the following theorem, with proof deferred to \pref{app: upper bound SSP}.  
\begin{theorem}
\label{thm:ssp_main_result}
    \pref{alg: cost case parameter free} guarantees for SSP problems that with probability at least $1-O(\delta)$, $\Reg_K=\tilde{O}(\sqrt{\Vinit \Bstar SAK} + \Bstar S^2A)$. 
\end{theorem}


%% file: conclusion.tex
\section{Conclusions and Open Problems}
In this work, we formulate the SP problem and give the first near-optimal regret bound for it. For special cases SLP and SSP, we further investigate the situation when the scale of the total reward $\Bstar$ is unknown. By improving previous adaptation results for SSP, and giving new lower bounds for SLP, we formally show a distinction between these two cases when $\Bstar$ is unknown. 

In the general case, although our algorithm achieves near-worst-case-optimal bounds in terms of $\Rinit$, there is still possibility of improving the bound using more refined quantities. We have ruled out possibility of $\Vinit, \Bstar$, and the possibility of $\Rstar$ when its value is unknown, but perhaps there are other candidate quantities. 
Further, there is a discrepancy in the analysis between the general SP / SLP setting and the SSP setting, i.e., while our result for the general case recovers that for the SLP case (up to logarithmic factors), it does not imply that for the SSP case. This also hints that $\Rinit$ does not always capture the true difficulty of every instance.

In general SP when $\Bstar$ is unknown, can we achieve the bound of order $\tilde{O}(\Rmax \sqrt{\poly(S,A)K})$, without any constant term that scales super-linearly with $\Rmax$? This would be analogous to the $\tilde{O}(\Bstar\sqrt{S^2AK})$ bound we get for SLP, but our technique there does not lead to this desired bound. 

Finally, it remains an open question to prove or disprove \citet{zhang2021reinforcement}'s conjecture about the lower-order term, which has direct consequences for our adaptivity result in SLP. 

%% file: appendix-general.tex
\section{Upper bounds for General Stochastic Path} \label{app: general upper}
\begin{definition}
    Let $Q_t, V_t$ be the $Q, V$ at the beginning of round $t$ (see the comments in \pref{alg: general case}). 
\end{definition}

\begin{definition}\label{def: tiliota}
    Define $\tiliota_{T,B,\delta}\triangleq (\ln(SA/\delta) + \ln\ln (BT))\times \ln T$. 
\end{definition}

\subsection{Optimism and regret decomposition}\label{app: optimism and regret decomp}

\begin{lemma}\label{lem: monotone}
    Define 
    \begin{align*}
        f(P, V, n, \iota)=PV  + \max\left\{c_1\sqrt{\frac{\mathbb{V}(P,V)\iota}{n}}, \frac{c_2 B\iota}{n}\right\}
    \end{align*}
    If $2c_1^2\leq c_2$ and $-B\leq V(\cdot)\leq B$, then $f(P, V, n, \iota)$ is increasing in $V$.  
\end{lemma}

\begin{proof}
    We compute the derivative of $f(P, V, n,\iota)$ over $V(s^\star)$: 
    \begin{align*}
        \frac{\partial f(P, V, n,\iota)}{\partial V(s^\star)} &= P(s^\star) 
        + \mathbf{1}\left\{c_1\sqrt{\frac{\mathbb{V}(P,V)\iota}{n}} >  \frac{c_2B\iota}{n}\right\}\times \frac{c_1 P(s^\star)(V(s^\star)-PV)\iota}{\sqrt{n\mathbb{V}(P,V)\iota}}\\
        &\geq P(s^\star) + \mathbf{1}\left\{c_1\sqrt{\frac{\mathbb{V}(P,V)\iota}{n}} >  \frac{c_2B\iota}{n}\right\}\times \frac{c_1^2P(s^\star)(V(s^\star)-PV)}{c_2B} \\
        &\geq P(s^\star) - \frac{2c_1^2}{c_2}P(s^\star) \\
        &\geq 0. 
    \end{align*}
\end{proof}

\begin{lemma}\label{lem: MVP}
    If $B\geq \Bstar$, then with probability at least $1-O(\delta)$, $Q_{t}(s,a)\geq Q^\star(s,a)$ for all $(s,a)$ and $t$. 
\end{lemma}
\begin{proof}
    We use induction to prove this. When $t=1$, $Q_1(s,a)=B\geq \Bstar\geq Q^\star(s,a)$ for all $s,a$. Suppose that $Q_t(s,a)\geq Q^\star(s,a)$ for all for all $s,a$ (which implies $V_t(s)\geq V^\star(s)$ for all $s$). Since $Q_{t+1}$ and $Q_t$ only differ in the entry $(s_t,a_t)$, we only need to check $Q_{t+1}(s_t,a_t)\geq Q^\star(s_t,a_t)$. This can be seen from the calculation below: 
    \begin{align*}
        Q_{t+1}(s_t,a_t)
        &=r(s_t,a_t) + \bar{P}_{t}V_t + b_t \\
        &= r(s_t,a_t) + \bar{P}_{t}V_t + \max\left\{c_1\sqrt{\frac{\mathbb{V}(\bar{P}_{t}, V_t)\iota_t}{n_{t}}}, \frac{c_2B\iota_t}{n_{t}}\right\}  \tag{by the induction hypothesis $V_t(\cdot)\geq V^\star(\cdot)$ and the monotone property \pref{lem: monotone}} \\
        &= r(s_t,a_t) + \bar{P}_{t}V^\star + \max\left\{c_1\sqrt{\frac{\mathbb{V}(\bar{P}_{t}, V^\star)\iota_t}{n_{t}}}, \frac{c_2B\iota_t}{n_{t}}\right\} \\
        &\geq r(s_t,a_t) + P_{s_t,a_t}V^\star  \tag{By \pref{lem: lnln concentration}}\\
        &= Q^\star(s_t,a_t).  
    \end{align*}
    Note that we can apply \pref{lem: lnln concentration} only for $n_t \geq  4\iota_t$. If $n_t<4\iota_t$, then the bias term itself is bounded by $\frac{c_2B}{4}\geq B$ which ensures optimism.
    This finishes the induction. 
\end{proof}
\begin{lemma}\label{lem: V*-Q*}
    Suppose that $B\geq \Bstar$. With probability at least $1-O(\delta)$, 
    \begin{align*}
        \sum_{t=1}^{T} \left(V^\star(s_t) - Q^\star(s_t,a_t) \right) 
        &\leq  O\left( \sqrt{SA\sum_{t=1}^{T} \mathbb{V}(P_t, V^\star)\tilde{\iota}_{T,B,\delta}} + B S^2 A\tilde{\iota}_{T,B,\delta}\right), 
    \end{align*}
    where $\tiliota_{T,B,\delta}$ is a logarithmic term defined in \pref{def: tiliota}. 
\end{lemma}

\begin{proof}
    \allowdisplaybreaks
    Below, we denote $P_t\triangleq P_{s_t,a_t}$.
    \begin{align*}
        &\sum_{t=1}^T \left(Q_t(s_t,a_t) - Q^\star(s_t,a_t) \right) \\
        &=\sum_{t=1}^T \left(Q_{t+1}(s_t,a_t) - Q^\star(s_t,a_t) \right) + \sum_{t=1}^T \left(Q_t(s_t,a_t) - Q_{t+1}(s_t,a_t)\right) \\
        &\leq \sum_{t=1}^T \left(\bar{P}_t V_t - P_t V^\star \right) + \sum_{t=1}^T b_t + \sum_{t=1}^T \sum_{s,a} (Q_t(s,a)-Q_{t+1}(s,a)) \\
        &\leq \sum_{t=1}^T P_t(V_t - V^\star) + \sum_{t=1}^T (\bar{P}_t - P_t)V^\star + \sum_{t=1}^T (\bar{P}_t - P_t)(V_t - V^\star) + \sum_{t=1}^T b_t + O(BSA)\\
        &\leq \underbrace{\sum_{t=1}^T  \mathbf{1}_{s_t'}(V_t - V^\star)}_{\term_1} + \underbrace{\sum_{t=1}^T (P_t-\mathbf{1}_{s_t'})(V_t - V^\star)}_{\term_2}
       + \underbrace{\sum_{t=1}^T ( \bar{P}_t - P_t)V^\star}_{\term_3} + \underbrace{\sum_{t=1}^T (\bar{P}_t - P_t )(V_t - V^\star)}_{\term_4} \\
       &\qquad \qquad + \underbrace{\sum_{t=1}^T b_t}_{\term_5}+ O(BSA)
    \end{align*}
    where the first inequality is because $Q_{t+1}(s_t,a_t)\leq r(s_t,a_t) + \bar{P}_tV_t + b_t$ and $Q^\star(s_t,a_t)=r(s_t,a_t) + P_tV^\star$, and that $Q_t(s,a)\geq Q_{t+1}(s,a)$. Below, we bound the individual terms. 
    \allowdisplaybreaks
    \begin{align}
        \term_1&=\sum_{t=1}^T  \mathbf{1}_{s_t'}(V_t - V^\star) \nonumber \\
        &\leq \sum_{t=1}^T  \mathbf{1}_{s_{t+1}}(V_t - V^\star)  \tag{$V_t(g)-V^\star(g)=0$ and $V_t(s)-V^\star(s)\geq 0$ by \pref{lem: MVP}} \\ 
        &\leq \sum_{t=1}^T  \mathbf{1}_{s_t}(V_t - V^\star) + \sum_{t=1}^T (V_{t}(s_{t+1}) - V_{t}(s_t)) + \sum_{t=1}^T  (V^\star(s_t) - V^\star(s_{t+1}))  \nonumber \\
        &\leq \sum_{t=1}^T  \mathbf{1}_{s_t}(V_t - V^\star) + \sum_{t=2}^T\sum_s (V_{t-1}(s) - V_t(s))  + O(B) \tag{$V_t$ decreases with $t$}  \nonumber \\
        &\leq \sum_{t=1}^T  \mathbf{1}_{s_t}(V_t - V^\star) +  O(BS).   \label{eq: term1 bound}
    \end{align}
    We have $(P_t-\mathbf{1}_{s_t'})(V_t-V^\star)$ conditioned on step $t$ is a zero mean random variable bounded in $[-B,B]$.
    By \pref{lem: freedman}, with probability at least $1-O(\delta)$, we have
    \begin{align*}
        \term_2 &= \sum_{t=1}^T (P_t-\mathbf{1}_{s_t'})(V_t - V^\star) = O\left(\sqrt{\sum_{t=1}^T \mathbb{V}(P_t, V_t-V^\star)\iota_t} + B\iota_t \right)\,.
    \end{align*}
    We have $(\bar P_t-P_t)V^\star = \frac{1}{n_t}\sum_{r: (s_r,a_r)=(s_t,a_t)}(\mathbf{1}_{s_r'}-P_r)V^\star$, which again by  \pref{lem: freedman} is bounded for a fixed state, action pair simultaneously over all time-steps. Via union bound over all states and actions, we have with probability $1-\mathcal{O}(\delta)$
    \begin{align}
        \term_3 &= \sum_{t=1}^T  (\bar{P}_t - P_t) V^\star = O\left(\sum_{t=1}^T \left(\sqrt{\frac{\mathbb{V}(P_t, V^\star)\iota_t}{n_t}} + \frac{\Bstar\iota_t}{n_t}\right)\right)\,. \label{eq: (P-P)V*}  
\end{align}
For the next term, we use a union bound over all state action pairs and apply \pref{lem: freedman over ball}, to obtain with probability $1-O(\delta)$ 
\begin{align}
        \term_4 &= \sum_{t=1}^T  (\bar{P}_t - P_t) (V_t-V^\star) = O\left(\sum_{t=1}^T \left(\sqrt{\frac{\mathbb{V}(P_t, V_t - V^\star)\iota_t}{n_t}} + \frac{B \iota_t}{n_t}\right)\right).  \label{eq: (P-P)(Vt-V*)}
    \end{align}
    Next, using \pref{lem: freedman over ball} again, we have with probability $1-O(\delta)$
    \begin{align*}
        \mathbb{V}(\bar P_t, V_t) &= \bar P_t (V_t-\bar P_t V_t)^2\\ &\leq \bar P_t (V_t- P_t V_t)^2\\
&= \mathbb{V}(P_t, V_t) +(\bar P_t-P_t) (V_t- P_t V_t)^2 \\
&\leq \mathbb{V}(P_t, V_t) +2B|(\bar P_t-P_t) (V_t- P_t V_t)|\\
&\leq \mathbb{V}(P_t, V_t) +O\left(B\sqrt{S\frac{\mathbb{V}(P_t,V_t)\iota_t}{n_t}} + \frac{ SB^2\iota_t}{n_t}\right)\\
&\leq O\left(\mathbb{V}(P_t, V_t)+  \frac{ SB^2\iota_t}{n_t}\right)\,.\tag{AM-GM inequality}
\end{align*}
    By the definition of $b_t$, with probability at least $1-O(\delta)$, 
    \begin{align*}
        \term_5 = \sum_{t=1}^T b_t &= O\left(\sum_{t=1}^T \left(\sqrt{\frac{\mathbb{V}(\bar{P}_t, V_t)\iota_t}{n_t}} + \frac{B\iota_t}{n_t}\right)\right) \\ 
        &= O\left(\sum_{t=1}^T \left(\sqrt{\frac{\mathbb{V}(P_t, V_t)\iota_t}{n_t}} + \frac{B\sqrt{S}\iota_t}{n_t}\right)\right)   
        \\  
        &= O\left(\sum_{t=1}^T \left(\sqrt{\frac{\mathbb{V}(P_t, V^\star)\iota_t}{n_t}} + \sqrt{\frac{\mathbb{V}(P_t, V_t - V^\star)\iota_t}{n_t}} + \frac{B\sqrt{S}\iota_t}{n_t}\right)\right)
    \end{align*}
    Collecting terms and using Cauchy-Schwarz, we get 
    \begin{align*}
        &\sum_{t=1}^T \left(Q_t(s_t,a_t) - Q^\star(s_t,a_t) \right) \leq \sum_{t=1}^T (V_t(s_t) - V^\star(s_t)) \\
        &\qquad + O\left( \sqrt{SA\sum_{t=1}^T \mathbb{V}(P_t, V^\star)\tilde{\iota}_{T,B,\delta}} + \sqrt{S^2A\sum_{t=1}^T  \mathbb{V}(P_t, V_t - V^\star)\tilde{\iota}_{T,B,\delta}}  + B S^2 A\tilde{\iota}_{T,B,\delta}\right). 
    \end{align*}
    We further invoke \pref{lem: sum of variance term 1} and bound the last expression by 
    \begin{align*}
        \sum_{t=1}^T (V_t(s_t) - V^\star(s_t)) + O\left( \sqrt{SA\sum_{t=1}^T \mathbb{V}(P_t, V^\star)\tilde{\iota}_{T,B,\delta}} + B S^2 A\tilde{\iota}_{T,B,\delta}\right). 
    \end{align*}
    Finally, noticing that $Q_t(s_t,a_t)=V_t(s_t)$ by the choice of $a_t$ finishes the proof. 
\end{proof}

\begin{lemma}\label{lem: sum of variance term 1}
With probability at least $1-O(\delta)$, 
\begin{align*}
    \sum_{t=1}^T  \mathbb{V}(P_t, V_t - V^\star) 
    &= O\left(\frac{1}{S}\sum_{t=1}^T  \mathbb{V}(P_t, V^\star) + B^2 S^2 A\tilde{\iota}_{T,B,\delta} \right).  
\end{align*}  
\end{lemma}
\allowdisplaybreaks
\begin{proof}
    Using \pref{lem: X lemma} with $X_t = V_t - V^\star$, we get
    \begin{align*}
        &\sum_{t=1}^T  \mathbb{V}(P_t, V_t - V^\star) \\ 
        &= O\left(B\sum_{t=1}^T |V_t(s_t) - V^\star(s_t) - P_t(V_t - V^\star)| + B\sum_{t=1}^T\sum_s|V_t(s) - V_{t+1}(s)| + B^2\ln(1/\delta) \right) \\
        &= O\left(B\sum_{t=1}^T |(\bar{P}_t - P_t)V_t| + B\sum_{t=1}^T b_t + B^2S\ln(1/\delta) \right) \\
        &\leq O\left( B\sqrt{SA\sum_{t=1}^T \mathbb{V}(P_t, V^\star)\tilde{\iota}_{T,B,\delta}} + B\sqrt{S^2A\sum_{t=1}^T  \mathbb{V}(P_t,  V_t - V^\star)\tilde{\iota}_{T,B,\delta}}  + B^2 S^2 A\tilde{\iota}_{T,B,\delta} \right). \tag{by the same argument as in \pref{eq: (P-P)V*} and \pref{eq: (P-P)(Vt-V*)}}
    \end{align*}
    Solving for $\sum_{t=1}^T \mathbb{V}(P_t,V_t-V^\star)$, we get 
    \begin{align*}
        \sum_{t=1}^T \mathbb{V}(P_t,V_t-V^\star) 
        &= O\left(B\sqrt{SA\sum_{t=1}^T \mathbb{V}(P_t, V^\star)\tilde{\iota}_{T,B,\delta}} + B^2S^2A\tilde{\iota}_T\right) \\
        &= O\left(\frac{1}{S}\sum_{t=1}^T \mathbb{V}(P_t, V^\star) + B^2S^2A\tilde{\iota}_{T,B,\delta}\right).    \tag{by AM-GM}
    \end{align*}
\end{proof}

    
   

\subsection{Bounding the sum of variance}
In \pref{app: optimism and regret decomp}, we have already shown that 
\begin{align*}
    \sum_{t=1}^T (V^\star(s_t) - Q^\star(s_t,a_t)) = \tilde{O}\left( \sqrt{SA\sum_{t=1}^T \mathbb{V}(P_t,V^\star)} + BS^2A \right)
\end{align*}
in \pref{lem: V*-Q*}. 
In this subsection (\pref{lem: getting a high prob bound}), we close the loop and show 
\begin{align}
    \sum_{t=1}^T \mathbb{V}(P_t, V^\star) = \tilde{O}\left(R^2 K + \Rmax \sum_{t=1}^T (V^\star(s_t) - Q^\star(s_t,a_t)) + \Rmax^2\right). \label{eq: close the loop}
\end{align}
We first establish some useful properties. 
\begin{lemma}\label{lem: several useful ones}
    Define the following notation: 
\begin{align}
       Y_k \triangleq \sum_{t=t_k}^{e_k} \mathbb{V}(P_t, V^\star),\quad 
       Z_k \triangleq \sum_{t=t_k}^{e_k} (V^\star(s_t) - Q^\star(s_t,a_t)), \label{eq: YZ notation}
\end{align}
    and recall that $\ln_+(x)\triangleq \ln(1+x)$. We have 
    \begin{align}
        \E_{t_k}[Z_k^2] &\leq O\left(\Rmax\ln_+\left({\frac{\Rmax}{R}}\right)\E_{t_k}[Z_k] + 1\right),   \label{eq: prop1}\\
        \E_{t_k}[Y_k] &\leq O\left(\Rmax\ln_+\left({\frac{\Rmax}{R}}\right)\E_{t_k}[Z_k] + R^2 \right), \label{eq: prop2}\\
        \E_{t_k}[Y_k^2] &\leq O\left(\Rmax^3\ln_+^2\left(\frac{\Rmax}{R}\right) \E_{t_k}\left[Z_k\right] + \Rmax^2R^2\ln_+\left({\frac{\Rmax}{R}}\right)\right),  \label{eq: prop3} \\
        Z_k &\leq \Rmax \ln\left(\frac{K}{\delta}\right) \ \text{for all $k\in[K]$ w.p.} \geq 1-\delta, \label{eq: prop4}\\
        Y_k &\leq \Rmax^2\ln\left(\frac{K}{\delta}\right)\ \text{for all $k\in[K]$ w.p.} \geq 1-\delta.  \label{eq: prop5}
    \end{align}
\end{lemma}
\allowdisplaybreaks
Before proving \pref{lem: several useful ones}, we point out that the key is to show \pref{eq: prop2}, which, after summing over $k$, will almost imply \pref{eq: close the loop}, but only \emph{in expectation}. The other inequalities \pref{eq: prop1}, \pref{eq: prop3}, \pref{eq: prop4}, \pref{eq: prop5} will be used in concentration inequalities that boost the expectation bound to a high-probability bound.

\begin{proof}[\pref{lem: several useful ones}]
   \paragraph{Proving \pref{eq: prop1}}
   We use \pref{lem: sum of positive} with $X_t = V^\star(s_t) - Q^\star(s_t,a_t)$. First, note that $0\leq V^\star(s_t) - Q^\star(s_t,a_t)\leq 2\Bstar\leq 2\Rmax$. Then note that for any $t'$ in episode $k$, 
    \begin{align}
        \E_{t'}\left[\sum_{t=t'}^{e_k}(V^\star(s_t) - Q^\star(s_t,a_t))\right] 
        &= \E_{t'}\left[\sum_{t=t'}^{e_k} (V^\star(s_t) - r(s_t,a_t) - P_tV^\star)\right] \nonumber \\
        &= \E_{t'}\left[\sum_{t=t'}^{e_k} (V^\star(s_t) - r(s_t,a_t) - V^\star(s_{t}'))\right] \nonumber \\
        &= V^\star(s_{t'}) - \E_{t'}\left[\sum_{t=t'}^{e_k} r(s_t,a_t)\right] \leq 2\Rmax. \label{eq: bound one episode regret}
    \end{align}
    Combining these two arguments and using \pref{lem: sum of positive}~(b) (with $c$ set to $R$), we get 
    \begin{align}
        \E_{t_k}[Z_k^2] &= \mathbb{E}_{t_k}\left[\left(\sum_{t=t_k}^{e_k} (V^\star(s_t) - Q^\star(s_t,a_t))\right)^2  \right]  \nonumber \\
        &\leq O\left(\Rmax\ln_+\left({\frac{\Rmax}{R}}\right) \E_{t_k}\left[\sum_{t=t_k}^{e_k} (V^\star(s_t) - Q^\star(s_t,a_t)) \right] + R^2\right) \nonumber \\
        &=O\left(\Rmax\ln_+\left({\frac{\Rmax}{R}}\right)\E_{t_k}[Z_k]+R^2\right). \nonumber
    \end{align} 
    \paragraph{Proving \pref{eq: prop2}} 
    Observe that 
    \begin{align}
        \E_{t_k}\left[\sum_{t=t_k}^{e_k}\mathbb{V}(P_t, V^\star)\right] = \E_{t_k}\left[\sum_{t=t_k}^{e_k} (V^\star(s_t') - P_tV^\star)^2 \right] = \mathbb{E}_{t_k}\left[\left(\sum_{t=t_k}^{e_k} (V^\star(s_t') - P_tV^\star)\right)^2 \right], \label{eq: converting}
    \end{align}
    where the last equality is because 
    \begin{align*}
        \E_{t_k}\left[ (V^\star(s_t') - P_tV^\star )(V^\star(s_u') - P_uV^\star) \right] = 0
    \end{align*}
    for any $u>t\geq t_k$. We continue with the following: 
    \begin{align}
        &\mathbb{E}_{t_k}\left[\left(\sum_{t=t_k}^{e_k} (V^\star(s_t') - P_tV^\star)\right)^2 \right] \nonumber \\ &= \mathbb{E}_{t_k}\left[\left(\sum_{t=t_k}^{e_k} (V^\star(s_t) - P_tV^\star) - V^\star(s_{t_k})\right)^2 \right] \nonumber \\
        &= \mathbb{E}_{t_k}\left[\left(\sum_{t=t_k}^{e_k} (V^\star(s_t) - Q^\star(s_t,a_t) + r(s_t,a_t)) - V^\star(s_{t_k})\right)^2 \right] \nonumber \\
        &\leq 3\mathbb{E}_{t_k}\left[\left(\sum_{t=t_k}^{e_k} (V^\star(s_t) - Q^\star(s_t,a_t))\right)^2  \right] + 3\mathbb{E}_{t_k}\left[\left(\sum_{t=t_k}^{e_k} r(s_t,a_t)\right)^2\right] + 3V^\star(s_{t_k})^2 \nonumber \\
        &\leq 3\mathbb{E}_{t_k}\left[\left(\sum_{t=t_k}^{e_k} (V^\star(s_t) - Q^\star(s_t,a_t))\right)^2  \right] + 6R^2.   \label{eq: translating between variance and regret}
    \end{align}
    Thus, combining the arguments above, we have proven $\E_{t_k}[Y_k]\leq O\left(\E_{t_k}[Z_k^2] + R^2\right)$. Further combining this with \pref{eq: prop1}, we get \pref{eq: prop2}. 
    \paragraph{Proving \pref{eq: prop3}} 
    For any $t'$ in episode $k$, by the same calculation as in \pref{eq: converting}, \pref{eq: translating between variance and regret} (but instead of starting time from $t_k$, start from an arbitrary $t'$ in episode $k$), we have 
    \begin{align}
        \E_{t'}\left[\sum_{t=t'}^{e_k}\mathbb{V}(P_t, V^\star)\right]&\leq O\left(\mathbb{E}_{t'}\left[\left(\sum_{t=t'}^{e_k} (V^\star(s_t) - Q^\star(s_t,a_t))\right)^2  \right] + R_{\max}^2\right) \leq O(\Rmax^2),   \label{eq: partial sum 2}
    \end{align}
    where the last inequality is by \pref{eq: bound one episode regret} and \pref{lem: sum of positive}~(c) with $X_t=V^\star(s_t)-Q^\star(s_t,a_t)$. Thus, 
    \begin{align}
        \E_{t_k}[Y_k^2] &= \E_{t_k}\left[\left(\sum_{t=t_k}^{e_k} \mathbb{V}(P_t,V^\star)\right)^2\right]  \nonumber \\
        &\leq O\left(\Rmax^2\ln_+\left(\frac{\Rmax^2}{R^2}\right)\E_{t_k}\left[\sum_{t=t_k}^{e_k}\mathbb{V}(P_t,V^\star)\right]+R^4\right)  \tag{by \pref{eq: partial sum 2} and \pref{lem: sum of positive}~(b) with $c=R^2$} \\
        &= O\left(\Rmax^3\ln_+^2\left(\frac{\Rmax}{R}\right) \E_{t_k}\left[Z_k\right] + \Rmax^2R^2\ln_+\left(\frac{\Rmax}{R}\right)\right). \tag{by \pref{eq: prop2}}  
    \end{align}
    \paragraph{Proving \pref{eq: prop4}}  This is directly by \pref{eq: bound one episode regret} and \pref{lem: sum of positive}~(a). 
    \paragraph{Proving \pref{eq: prop5}} This is directly by \pref{eq: partial sum 2} and \pref{lem: sum of positive}~(a). 
\end{proof}

\begin{lemma} \label{lem: getting a high prob bound}
    With probability at least $1-O(\delta)$, 
    \begin{align*}
        &\sum_{t=1}^T \mathbb{V}(P_t, V^\star)
        \leq \\ &O\left(\Rmax\ln_+\left({\frac{\Rmax}{R}}\right)\sum_{t=1}^T (V^\star(s_t) - Q^\star(s_t,a_t)) + R^2K + \Rmax^2\ln\left(\frac{\Rmax K}{R\delta}\right)\ln\left(\frac{\ln(\Rmax K)}{\delta}\right)\right). 
    \end{align*}
\end{lemma}

\begin{proof}
    Similar to \pref{eq: YZ notation}, we define 
    \begin{align*}
       Y_k \triangleq \sum_{t=t_k}^{e_k} \mathbb{V}(P_t, V^\star),\quad 
       Z_k \triangleq \sum_{t=t_k}^{e_k} (V^\star(s_t) - Q^\star(s_t,a_t)).
\end{align*}
    \allowdisplaybreaks
    By \pref{lem: freedman}, with probability at least $1-O(\delta)$, 
    \begin{align}
        &\sum_{k=1}^K Y_k  \nonumber \\
        &\leq \sum_{k=1}^K \E_{t_k}\left[Y_k\right] + O\left(\sqrt{\sum_{k=1}^K \E_{t_k}[Y_k^2] \ln\frac{\ln\left(\sum_{k=1}^K \E_{t_k}[Y_k^2]\right)}{\delta}} + \left(\max_{k\in[K]} Y_k\right) \ln\frac{\ln\left(\max_{k\in[K]}Y_k\right)}{\delta} \right)    \nonumber\\
        &\leq O\left(\Rmax \ln_+\left({\frac{\Rmax}{R}}\right) \sum_{k=1}^K \E_{t_k}\left[Z_k \right] + R^2K \right) \tag{by \pref{eq: prop2}} \\
        &\quad + O\left(\sqrt{\left(\Rmax^3\ln_+^2\left(\frac{\Rmax}{R}\right) \sum_{k=1}^K \E_{t_k}\left[Z_k\right] + \Rmax^2R^2\ln_+\left({\frac{\Rmax}{R}}\right)K\right)\ln\frac{\ln(\Rmax K\sum_{k=1}^K \E_{t_k}[Z_k])}{\delta}}\right) \tag{by \pref{eq: prop3}}\\
        &\quad + O\left( \Rmax^2 \ln\frac{K}{\delta}\ln\frac{\ln\left(\Rmax^2 \ln\frac{K}{\delta}\right)}{\delta} \right)  \tag{by \pref{eq: prop5}}\\
        &\leq O\left(\Rmax \ln_+\left({\frac{\Rmax}{R}}\right) \sum_{k=1}^K \E_{t_k}\left[Z_k \right] + R^2K +  \Rmax^2\ln\frac{\Rmax K}{R\delta}\ln\frac{\ln(\Rmax K)}{\delta}\right).  \tag{AM-GM and that $\E_{t_k}[Z_k]\leq \Rmax$} \\
        &\ \ \label{eq: Y_k concentration ineql}
    \end{align}
    Next, we connect $\sum_k \E_{t_k}[Z_k]$ with $\sum_k Z_k$.  By \pref{lem: freedman}, with probability at least $1-O(\delta)$, 
    \begin{align*}
        &\sum_{k=1}^K \E_{t_k}[Z_k] \\ &\leq \sum_{k=1}^K  Z_k + O\left(\sqrt{\sum_{k=1}^K \E_{t_k}[Z_k^2]\ln\frac{\ln\left(\sum_{k=1}^K \E_{t_k}[Z_k^2]\right)}{\delta}} + \left(\max_{k\in[K]} Z_k\right)\ln\frac{\ln\left(\max_{k\in[K]} Z_k\right)}{\delta} \right) \\
        &\leq \sum_{k=1}^K  Z_k + O\left(\sqrt{\Rmax\ln_+\left({\frac{\Rmax}{R}}\right)\sum_{k=1}^K \E_{t_k}[Z_k] \ln\frac{\ln(\Rmax K)}{\delta}} \right) \tag{by \pref{eq: prop1} and that $\E_{t_k}[Z_k^2]\leq \Rmax^2$}\\
        &\qquad + O\left(\Rmax\ln\frac{K}{\delta}\ln\frac{\ln(\Rmax\ln\frac{K}{\delta})}{\delta}\right)   \tag{by \pref{eq: prop4}}\\
        &\leq \sum_{k=1}^K  Z_k + \frac{1}{2}\sum_{k=1}^K \E_{t_k}[Z_k]  + O\left( \Rmax\ln\frac{\Rmax K}{R\delta}\ln\frac{\ln(\Rmax K)}{\delta}\right). \tag{AM-GM}
    \end{align*}
    Solving for $\sum_{k=1}^K \E_{t_k}[Z_k]$ and plugging it to \pref{eq: Y_k concentration ineql}, we get that with probability at least $1-O(\delta)$, 
    \begin{align*}
        \sum_{k=1}^K Y_k \leq O\left(\Rmax\ln_+\left({\frac{\Rmax}{R}}\right)\sum_{k=1}^K Z_k + R^2K + \Rmax^2\ln\frac{\Rmax K}{R\delta}\ln\frac{\ln(\Rmax K)}{\delta}\right). 
    \end{align*}
    This finishes the proof. 

\end{proof}

\subsection{Bounding the regret }
\begin{proof}[\pref{thm: main for mixed case}]
We use the following notations:     \begin{align*}
       Y_k \triangleq \sum_{t=t_k}^{e_k} \mathbb{V}(P_t, V^\star),\quad 
       Z_k \triangleq \sum_{t=t_k}^{e_k} (V^\star(s_t) - Q^\star(s_t,a_t)). 
\end{align*}
   \begin{align}
       &\Reg_K 
       =\sum_{k=1}^K \left( V^\star(s_{t_k}) - \sum_{t=t_k}^{e_k}r(s_t, a_t)\right)   \nonumber \\
       &= \sum_{t=1}^T (  V^\star(s_t) - V^\star(s_t') - r(s_t, a_t) ) \nonumber  \\
       &= \sum_{t=1}^T \left( V^\star(s_t) - Q^\star(s_t,a_t) \right) + \sum_{t=1}^T (P_tV^\star - V^\star(s_t')) \nonumber \\
       &\leq \sum_{t=1}^T \left( V^\star(s_t) - Q^\star(s_t,a_t)\right) + O\left(\sqrt{\sum_{t=1}^T \mathbb{V}(P_t, V^\star) \tilde{\iota}_{T,B,\delta}} + \Bstar\tilde{\iota}_{T,B,\delta} \right) \nonumber \\
       &\leq O\left(\sum_{k=1}^K Z_k + \sqrt{\sum_{k=1}^K Y_k \tilde{\iota}_{T,B,\delta} } + \Bstar\tilde{\iota}_{T,B,\delta}\right) \nonumber \\
       &\leq O\left(\sum_{k=1}^K Z_k + \sqrt{\left(\Rmax\ln_+\left(\frac{\Rmax}{R}\right)\sum_{k=1}^K Z_k + R^2K + \Rmax^2\ln\frac{\Rmax K}{R\delta}\ln\frac{\ln(\Rmax K)}{\delta}\right)\tilde{\iota}_{T,B,\delta}} + \Bstar \tilde{\iota}_{T,B,\delta}\right) \tag{\pref{lem: getting a high prob bound}} \\
       &\leq O\left(\sum_{k=1}^K Z_k + R\sqrt{K\tilde{\iota}_{T,B,\delta}} + \Rmax \ln\left(\frac{\Rmax K}{R\delta}\right)\tilde{\iota}_{T,B,\delta}\right).   \tag{AM-GM} \\  
       &  \ \ \label{eq: regret final}
   \end{align}
   It remains to bound $\sum_{k=1}^K Z_k$. 
   By \pref{lem: V*-Q*}, we have with probability at least $1-O(\delta)$, 
   \begin{align}
       \sum_{k=1}^K Z_k \leq O\left( \sqrt{SA\sum_{k=1}^K Y_k \tilde{\iota}_{T,B,\delta}} + BS^2A \tilde{\iota}_{T,B,\delta}\right).  
   \end{align}
   Further using \pref{lem: getting a high prob bound} on the right-hand side,  
   \begin{align*}
       &\sum_{k=1}^K Z_k \leq \\
       &O\left(  \sqrt{SA\left(\Rmax\ln_+\left(\frac{\Rmax}{R}\right)\sum_{k=1}^K Z_k + R^2K + \Rmax^2\ln\frac{\Rmax K}{R\delta}\ln\frac{\ln(\Rmax K)}{\delta}\right)\tilde{\iota}_{T,B,\delta}} + BS^2A \tilde{\iota}_{T,B,\delta}\right). 
   \end{align*}
   Solving for $\sum_{k=1}^K Z_k$, we get 
   \begin{align*}
       \sum_{k=1}^K Z_k \leq O\left(R\sqrt{SAK\tilde{\iota}_{T,B,\delta}} + \Rmax SA\ln\left(\frac{\Rmax K}{R\delta}\right)\tilde{\iota}_{T,B,\delta} + BS^2A\tilde{\iota}_{T,B,\delta}\right). 
   \end{align*}
   Plugging this to \pref{eq: regret final} finishes the proof. 
\end{proof}

%% file: appendix-reward.tex
\section{Upper Bound for Stochastic Longest Path}\label{app: upper SLP}
\begin{proof}[\pref{lem: connecting reward and general}]
Define 
\begin{align*}
    B_{\pi} &= \max\left\{\sup_s \E^\pi\left[\sum_{t=1}^\infty r(s_t,a_t)~\Big|~ s_1=s \right], 1 \right\}\\
    V_\pi &= \max \left\{\E^\pi\left[\sum_{t=1}^\infty r(s_t,a_t)~\Big|~ s_1=s_\init\right], 1\right\}
\end{align*}
Since $r(\cdot,\cdot)\geq 0$, by \pref{lem: sum of positive}~(b) (with $c\triangleq \min\{B_\pi, \Vinit\}$), we have \begin{align*}
   R^2 &\leq O\left(\sup_{\pi}V_\pi B_\pi\ln_+ \left(\frac{B_\pi}{\min\{B_\pi, \Vinit\}}\right) + \Vinit^2 \right) \leq O\left(\Vinit \Bstar \ln_+\frac{\Bstar}{\Vinit}\right). 
\end{align*}
By \pref{lem: sum of positive}~(c), we have
\begin{align*}
    \Rmax^2 \leq O\left(\sup_\pi B_\pi^2 \right) \leq O\left(\Bstar^2\right). 
\end{align*}
\end{proof}

\begin{lemma}\label{lem: cost sm of var}
    With probability at least $1-\delta$, for all $T$, 
    \begin{align*}
        \sum_{t=1}^T \mathbb{V}(P_t, V^\star) \leq O\left(\Bstar\sum_{t=1}^T (V^\star(s_t) - Q^\star(s_t,a_t)) + \Bstar\sum_{t=1}^T |r(s_t,a_t)| + \Bstar^2\ln(T/\delta)\right). 
    \end{align*}
\end{lemma}

\begin{proof}
    \begin{align*}
        &\sum_{t=1}^T  \mathbb{V}(P_t, V^\star) \\
        &= \sum_{t=1}^T \left(\E_{s'\sim P_t}[V^\star(s')^2] - (P_t V^\star)^2\right)\\
        &= \sum_{t=1}^T (V^\star(s_t')^2 - (P_tV^\star)^2) + \sum_{t=1}^T  (\E_{s'\sim P_t}[V^\star(s')^2] - V^\star(s_t')^2)  \\
        &\leq  \sum_{t=1}^T \left(V^\star(s_t)^2 - (P_tV^\star)^2\right)   + \Bstar^2 + \sum_{t=1}^T  (\E_{s'\sim P_t}[V^\star(s')^2] - V^\star(s_t')^2)  \tag{because $V^\star(s_t')^2\leq V^\star(s_{t+1})^2$}  \\
        &= \sum_{t=1}^T (V^\star(s_t)^2 - Q^\star(s_t,a_t)^2) + \sum_{t=1}^T (Q^\star(s_t,a_t)^2 - (P_tV^\star)^2) \\
        &\qquad \qquad + O\left(\sqrt{\sum_{t=1}^T \mathbb{V}(P_t,V^{\star 2})\ln(T/\delta)} + \Bstar^2 \ln(T/\delta)\right)  \\
        &\leq O\left(\Bstar \sum_{t=1}^T |V^\star(s_t) - Q^\star(s_t,a_t)| +  \Bstar\sum_{t=1}^T \left|Q^\star(s_t, a_t) - P_t V^\star\right| \right) \\
        &\qquad \qquad + O\left(\Bstar\sqrt{\sum_{t=1}^T \mathbb{V}(P_t,V^{\star})\ln(T/\delta)} + \Bstar^2 \ln(T/\delta)\right) \tag{$a^2-b^2\leq |a+b||a-b|$ and \pref{lem: lemma 30}} \\
        &\leq O\left(\Bstar\sum_{t=1}^T (V^\star(s_t) - Q^\star(s_t,a_t)) + \Bstar\sum_{t=1}^T |r(s_t,a_t)|\right) + \frac{1}{2}\sum_{t=1}^T \mathbb{V}(P_t, V^\star) + O\left(\Bstar^2 \ln(T/\delta)\right)   \tag{AM-GM}
    \end{align*}
    Solving for $\sum_{t=1}^T \mathbb{V}(P_t, V^\star)$ finishes the proof. 
\end{proof}

\begin{proof}[\pref{thm: main thm for reward}]
By the same calculation as in \pref{eq: regret final}, we have 
\begin{align*}
    \Reg_K \leq \sum_{t=1}^T \left( V^\star(s_t) - Q^\star(s_t,a_t)\right) + O\left(\sqrt{\sum_{t=1}^T \mathbb{V}(P_t, V^\star) \tilde{\iota}_{T,B,\delta}} + \Bstar\tilde{\iota}_{T,B,\delta} \right).
\end{align*}
Using \pref{lem: cost sm of var}, we get 
\begin{align}
    \Reg_K &\leq \sum_{t=1}^T (V^\star(s_t) - Q^\star(s_t,a_t)) \nonumber \\
    &\qquad + O\left( \sqrt{\Bstar \sum_{t=1}^T (V^\star(s_t) - Q^\star(s_t,a_t))\tiliota_{T,B,\delta} + \Bstar\sum_{t=1}^T r(s_t,a_t)\tiliota_{T,B,\delta} } + \Bstar \tiliota_{T,B,\delta}\right) \nonumber  \\
    &\leq O\left(\sum_{t=1}^T (V^\star(s_t) - Q^\star(s_t,a_t)) + \sqrt{\Bstar\sum_{t=1}^T r(s_t,a_t)\tiliota_{T,B,\delta} } + \Bstar \tiliota_{T,B,\delta} \right).   \label{eq: Reg_K temp bound}
\end{align}
By \pref{lem: V*-Q*}, 
\begin{align*}
    &\sum_{t=1}^T (V^\star(s_t) - Q^\star(s_t,a_t)) \\
    &\leq O\left(\sqrt{SA\sum_{t=1}^T  \mathbb{V}(P_t,V^\star)\tiliota_{T,B,\delta}} + BS^2A \tiliota_{T,B,\delta}\right) \\
    &\leq O\left( \sqrt{\Bstar SA\sum_{t=1}^T (V^\star(s_t) - Q^\star(s_t,a_t))\tiliota_{T,B,\delta} + \Bstar SA\sum_{t=1}^T r(s_t,a_t)\tiliota_{T,B,\delta} } + BS^2A \tiliota_{T,B,\delta}\right)
\end{align*}
where in the last inequality we again use \pref{lem: cost sm of var}. 
Solving for $\sum_{t=1}^T (V^\star(s_t) - Q^\star(s_t,a_t))$, we get 
\begin{align*}
    \sum_{t=1}^T (V^\star(s_t) - Q^\star(s_t,a_t)) \leq O\left(\sqrt{\Bstar SA\sum_{t=1}^T r(s_t,a_t)\tiliota_{T,B,\delta}} + BS^2A \tiliota_{T,B,\delta} \right).
\end{align*}
Using this in \pref{eq: Reg_K temp bound}, we get 
\begin{align*}
    K\Vinit - \sum_{t=1}^T  r(s_t,a_t)\leq O\left(\sqrt{\Bstar SA\sum_{t=1}^T r(s_t,a_t)\tiliota_{T,B,\delta}} + BS^2A \tiliota_{T,B,\delta} \right).
\end{align*}
If $\sum_{t=1}^T r(s_t,a_t) \geq K\Vinit$, we have $K\Vinit - \sum_{t=1}^T  r(s_t,a_t)\leq 0$; if 
$\sum_{t=1}^T r(s_t,a_t) \leq K\Vinit$, we can further bound the $\sum_{t=1}^T r(s_t,a_t)$ term on the right-hand side above by $K\Vinit$. In both cases, we have 
\begin{align*}
    K\Vinit - \sum_{t=1}^T  r(s_t,a_t)\leq O\left(\sqrt{\Vinit \Bstar SAK\tiliota_{T,B,\delta}} + BS^2A \tiliota_{T,B,\delta} \right). 
\end{align*}
\end{proof}

\begin{lemma}\label{lem: negative regret}
    Let $r(\cdot,\cdot)\geq 0$. With probability at least $1-\delta$, for all $K\geq 1$, with $T$ being the total number of steps in $K$ episodes, $$\Reg_K \geq -O\left( \sqrt{\Vinit \Bstar K\ln(T/\delta)} + \Bstar\ln(T/\delta)\right).$$ 
\end{lemma} 
\begin{proof}
    \begin{align*}
        \Reg_K 
        &= \sum_{k=1}^K \left(V^\star(s_{\init}) - \sum_{t=t_k}^{e_k}r(s_t,a_t)\right) \\
        &= \sum_{t=1}^T \left(V^\star(s_{t}) - V^\star(s_t') - r(s_t,a_t)\right) \\
        &= \sum_{t=1}^T \left(V^\star(s_t) - Q^\star(s_t,a_t)\right) + \sum_{t=1}^T (P_tV^\star - V^\star(s_t')) \\
        &\geq \sum_{t=1}^T \left(V^\star(s_t) - Q^\star(s_t,a_t)\right) - O\left(\sqrt{\sum_{t=1}^T \mathbb{V}(P_t,V^\star) \ln(T/\delta) } + \Bstar \ln(T/\delta)\right). 
    \end{align*}
    By \pref{lem: cost sm of var}, we can further lower bound the above expression by 
    \begin{align*}
        &\sum_{t=1}^T \left(V^\star(s_t) - Q^\star(s_t,a_t)\right) \\
        &\qquad  - O\left(\sqrt{\Bstar \sum_{t=1}^T (V^\star(s_t) - Q^\star(s_t,a_t))\ln(T/\delta) + \Bstar\sum_{t=1}^T r(s_t,a_t)\ln(T/\delta)} + \Bstar  \ln(T/\delta)\right) \\
        &\geq \sum_{t=1}^T \left(V^\star(s_t) - Q^\star(s_t,a_t)\right) - \frac{1}{2}\sum_{t=1}^T \left(V^\star(s_t) - Q^\star(s_t,a_t)\right) \\
        &\qquad - O\left(\sqrt{\Bstar\sum_{t=1}^T r(s_t,a_t)\ln(T/\delta)} + \Bstar\ln(T/\delta)\right)   \tag{AM-GM}\\
        &\geq - O\left(\sqrt{\Bstar\sum_{t=1}^T r(s_t,a_t)\ln(T/\delta)} + \Bstar\ln(T/\delta)\right). \tag{$V^\star(s_t) - Q^\star(s_t,a_t)\geq 0$}
    \end{align*}
    Hence we have 
    \begin{align*}
        \Reg_K = K\Vinit - \sum_{t=1}^T r(s_t,a_t) \geq - O\left(\sqrt{\Bstar\sum_{t=1}^T r(s_t,a_t)\ln(T/\delta)} + \Bstar\ln(T/\delta)\right).
    \end{align*}
    Solving for $\sum_{t=1}^T r(s_t,a_t)$, we get 
    \begin{align*}
        K\Vinit - \sum_{t=1}^T r(s_t,a_t) &\geq - O\left(\sqrt{\Vinit\Bstar K\ln(T/\delta)} + \Bstar\ln(T/\delta)\right). 
    \end{align*}
\end{proof}

\begin{proof}[\pref{lem: for the estimating procedure}]
    We consider a particular $i_\star\in\{1, \ldots, \lceil\log_2 U\rceil\}$ such that $2^{i-1} \leq \Vinit < 2^i$. By the assumption that $1\leq \Vinit \leq U$, such $i_\star$ always exists. Notice that in the $i_\star$-th for-loop, the input is $B=2^i\zeta \geq \Vinit \zeta \geq \Bstar$. Thus, according to \pref{thm: main thm for reward}, when the $i_\star$-th for-loop terminates, we have with probability at least $1-\delta'$, 
    \begin{align*}
        \Vinit - \hat{r}_{i_\star} 
        &\leq \frac{1}{N}\times c\times \left(\sqrt{\Vinit \Bstar SAN\tiliota_{M,B,\delta'}} + BS^2A\tiliota_{M,B,\delta'}\right) \\
        &= c\times\left(\sqrt{\frac{\Vinit \Bstar SA\tilde{\iota}_{M,B,\delta'}}{N}} + \frac{BS^2A\tilde{\iota}_{M,B,\delta'}}{N}\right) \\
        &\leq c\times \left(\sqrt{\frac{\Vinit\Bstar}{16c^2\zeta S}} + \frac{2^i\zeta }{16c^2 \zeta}\right) \tag{because the termination condition is $N\geq 16c^2\zeta S^2A\tiliota_{M,B,\delta'}$}\\
        &\leq \frac{1}{4}\Vinit + \frac{1}{8} \Vinit,   \tag{by the assumptions $\Bstar\leq \zeta \Vinit$ and $2^{i-1}\leq \Vinit$}
    \end{align*}
    which implies $\hat{r}_{i_\star}\geq \frac{1}{2}\Vinit$. 
    
    Next, we consider an arbitrary $i\in\{1, \ldots, \lceil\log_2 U\rceil\}$. Because $\Vinit$ is the expected reward of the optimal policy, in every episode, $\Vinit$ is larger than the expected reward of the learner. By \pref{lem: negative regret}, we have with probability at least $1-\delta'$, for all $N$,  
    \begin{align*}
        \Vinit - \hat{r}_i 
        &\geq -\frac{1}{N}\times c\times \left(\sqrt{\Vinit \Bstar N\ln(M/\delta')} + \Bstar S^2A\ln(M/\delta')\right) \\
        &\geq -\frac{1}{2}\Vinit, \tag{by the same calculation as above and noticing that $\ln(M/\delta')\leq \tiliota_{M,B,\delta'}$} 
    \end{align*}
    which implies $\hat{r}_i\leq \frac{3}{2}\Vinit$. With an union bound over $i$, the inequality holds for all $i$ with probability at least $1-\delta$. Combining the arguments, we conclude that with probability at least $1-2\delta$, 
    \begin{align*}
        \frac{1}{2}\Vinit \leq \max_i \{\hat{r}_i\} \leq \frac{3}{2}\Vinit
    \end{align*}
    The lemma is proven by noticing that $\hat{V} = 2\max_i\{\hat{r}_i\}$. 
\end{proof}

\begin{proof}[\pref{thm: main for parameter free reward setting}]
    We first consider the case when $\frac{\Bstar}{\Vinit}\leq \zeta$ and $\Vinit\leq U$.  
    
    Let $\{N_i\}_{i=1,2,\ldots, \lceil \log_2 U \rceil}$ be the number of episodes spent in the $i$-th for-loop in \pref{alg: estimating V*}. Thus, the total number of episodes the learner spends to estimate $\hat{V}$ is 
    \begin{align*}
        \sum_{i=1}^{\lceil \log_2 U \rceil} N_i = O\left(\sum_{i=1}^{\lceil \log_2 U\rceil}\zeta S^2A \tiliota_{M_i, 2^i\zeta, \delta'}\right)
    \end{align*}
    where $M_i$ is the number of steps spent in the $i$-th for-loop in \pref{alg: estimating V*}. By definition, $\tiliota_{M_i,2^i\zeta, \delta'}=O((\ln(SA/\delta') + \ln\ln(2^i\zeta M_i))\times \ln M_i) = O((\ln(SA/\delta) + \ln\ln T + \ln\ln (\zeta U))\times \ln T)=O(\tiliota_{T,\zeta U, \delta})$, and thus 
    \begin{align*}
        \sum_{i=1}^{\lceil \log_2 U \rceil} N_i = O\left(\zeta S^2 A \tiliota_{T,\zeta U, \delta} \times \ln U\right). 
    \end{align*}
    For these episodes, we simply bound the per-episode regret by $\Vinit$. By \pref{lem: for the estimating procedure}, $\hat{V}\in[\Vinit, 3\Vinit]$, and so the main algorithm (\pref{alg: general case}) is run with $B=\zeta \hat{V}\geq \zeta \Vinit \geq \Bstar$.  The regret incurred when running \pref{alg: general case} with $B=\zeta\hat{V}$, according to \pref{thm: main thm for reward}, is thus upper bound by 
    \begin{align*}
        O\left(\sqrt{\Vinit \Bstar SAK\tiliota_{T,\zeta U, \delta}} + \zeta \hat{V} S^2A\tiliota_{T,\zeta U, \delta}\right) \leq O\left(\sqrt{\Vinit \Bstar SAK\tiliota_{T,\zeta U, \delta}} + \zeta \Vinit S^2A\tiliota_{T,\zeta U, \delta}\right). 
    \end{align*}
    Overall, the regret (including the estimation part and the main algorithm part), is upper bounded by \begin{align*}
        O\left(\sqrt{\Vinit \Bstar SAK\tiliota_{T,\zeta U, \delta}} + \zeta \Vinit S^2A\tiliota_{T,\zeta U, \delta}\ln U\right). 
    \end{align*}
    We remind that this is the regret bound we can achieve when $\frac{\Bstar}{\Vinit}\leq \zeta$ and $\Vinit\leq U$. If either of them does not hold, we simply bound the total regret by $K\Vinit$. Hence, the overall regret without these two assumptions is upper bounded by 
    \begin{align*}
         &O\left(\sqrt{\Vinit \Bstar SAK\tiliota_{T,\zeta U, \delta}} + \zeta \Vinit S^2A\tiliota_{T,\zeta U, \delta}\ln U\right) + K\Vinit \ind{\frac{\Bstar}{\Vinit}>\zeta} + K\Vinit \ind{\Vinit > U}.  
    \end{align*}
    \paragraph{Case 1. $U$ is a known absolute upper bound for $\Vinit$} In this case, $\ind{\Vinit>U}=0$, and we have 
    \begin{align*}
        \Reg_K = O\left(\sqrt{\Vinit \Bstar SAK\tiliota_{T,\zeta U, \delta}} + \zeta \Vinit S^2A\tiliota_{T,\zeta U, \delta}\ln U\right) + K\Vinit \ind{\frac{\Bstar}{\Vinit}>\zeta}
    \end{align*}
    Let $\alpha>0$ be a parameter and set $\zeta = \sqrt{K}/\alpha$. Then the last expression can be upper bounded by 
    \begin{align*}
        &O\left(\sqrt{\Vinit \Bstar SAK\tiliota_{T,\zeta U, \delta}} + \frac{\Vinit S^2A(\ln U) \sqrt{K}\tiliota_{T,\zeta U, \delta}}{\alpha} + K\Vinit\ind{\frac{\Bstar}{\Vinit}> \frac{\sqrt{K}}{\alpha} }\right) \\
        &\leq O\left(\sqrt{\Vinit \Bstar SAK\tiliota_{T,\zeta U, \delta}} + \frac{\Vinit S^2A(\ln U) \sqrt{K}\tiliota_{T,\zeta U, \delta}}{\alpha} + \min\left\{ \alpha \Bstar\sqrt{K}, \frac{\alpha^2\Bstar^2}{\Vinit} \right\}\right)
    \end{align*}
    where in the last inequality we use two different ways to bound $K\Vinit$ under the inequality $\frac{\Bstar}{\Vinit}>\frac{\sqrt{K}}{\alpha}$: $K\Vinit\leq K\times \frac{\alpha \Bstar}{\sqrt{K}} = \alpha \Bstar\sqrt{K}$, and $K\Vinit \leq \left(\frac{\alpha\Bstar}{\Vinit}\right)^2\Vinit = \frac{\alpha^2 \Bstar^2}{\Vinit}$. If we set $\alpha=\sqrt{S^2A\ln U}$ and pick up the first term in $\min\{\cdot,\cdot\}$, then we get 
    \begin{align*}
        \Reg_K &= O\left(\sqrt{\Vinit \Bstar SAK\tiliota_{T,\zeta U, \delta}} + \Vinit \sqrt{S^2AK\ln U}\tiliota_{T,\zeta U, \delta} + \Bstar\sqrt{S^2A K\ln U} \right)\\
        &\leq O\left(\Bstar\sqrt{S^2AK\ln U}\tiliota_{T,KU,\delta}\right).
    \end{align*}
    If we set $\alpha=\sqrt{S^3A\ln U}$ and pick up the second term in $\min\{\cdot,\cdot\}$, then we get 
    \begin{align*}
        \Reg_K 
        &= O\left(\sqrt{\Vinit \Bstar SAK\tiliota_{T,\zeta U, \delta}} + \Vinit \sqrt{SAK\ln U}\tiliota_{T,\zeta U, \delta} + \frac{\Bstar^2S^3A \ln U}{\Vinit}\tiliota_{T,\zeta U,\delta}\right) \\
        &\leq O\left( \sqrt{\Vinit \Bstar SAK \ln U}\tiliota_{T,KU, \delta} + \frac{\Bstar^2S^3A \ln U}{\Vinit}\tiliota_{T,KU,\delta}\right). 
    \end{align*}
    \paragraph{Case 2. Unknown range of $\Vinit$ and set $U=K^{\frac{1}{\epsilon}}$}  By the same argument above, if we set $\zeta=\sqrt{K/(S^2A\ln U)}$, then we have 
    \begin{align*}
        \Reg_K &= 
        O\left(\Bstar\sqrt{S^2AK\ln U}\tiliota_{T,KU,\delta} + K\Vinit \ind{\Vinit > U}\right) \\
        &\leq O\left(\Bstar\sqrt{\epsilon^{-1}S^2AK\ln K}\tiliota_{T,K,\delta\epsilon} + K\Vinit \ind{\Vinit > K^{\frac{1}{\epsilon}}}\right) \\
        &\leq O\left(\Bstar\sqrt{\epsilon^{-1}S^2AK\ln K}\tiliota_{T,K,\delta\epsilon} + \Vinit^{1+\epsilon}\right),
    \end{align*}
    and if we set $\zeta = \sqrt{K/(S^3A\ln U)}$, then 
    \begin{align*}
        \Reg_K &= O\left( \sqrt{\Vinit \Bstar SAK \ln U}\tiliota_{T,KU, \delta} + \frac{\Bstar^2S^3A \ln U}{\Vinit}\tiliota_{T,KU,\delta}+ K\Vinit\ind{\Vinit>U}\right) \\
        &\leq O\left(\sqrt{ \epsilon^{-1}\Vinit \Bstar SAK \ln K}\tiliota_{T,K, \delta\epsilon} + \frac{\epsilon^{-1}\Bstar^2S^3A \ln K}{\Vinit}\tiliota_{T,K,\delta\epsilon}+ \Vinit^{1+\epsilon}\right).
    \end{align*}
\end{proof}

%% file: appendix-cost.tex
\section{Upper Bound for Stochastic Shortest Path}\label{app: upper bound SSP}

\begin{lemma}\label{lem: MVP for cost}
    \pref{alg: cost case parameter free} ensures $Q_{t}(s,a)\geq Q^\star(s,a)$ for all $(s,a)$ and $t$ with probability at least $1-\delta$. 
\end{lemma}
\begin{proof}
    For any $B>0$, we define $V^\star_{[B]}\in[-B, 0]^\calS$ to be such that 
    \begin{align*}
        V_{[B]}^\star(s) = \max\{-B, V^\star(s)\}. 
    \end{align*}

    We use induction to prove the lemma. When $t=1$, $Q_1(s,a)=0\geq Q^\star(s,a)$ for all $s,a$ since we are in the cost setting. Suppose that $0\geq Q_t(s,a)\geq Q^\star(s,a)$ for all $s,a$ (which implies $0\geq V_t(s)\geq V^\star(s)$ for all $s$). Since $Q_{t+1}$ and $Q_t$ only differ in the entry $(s_t,a_t)$, we only need to check $Q_{t+1}(s_t,a_t)\geq Q^\star(s_t,a_t)$. With the definition of $b_t$ and $\iota_t$ specified in \pref{line: bt and iotat position} of \pref{alg: cost case parameter free}, we have
    \begin{align*}
        Q_{t+1}(s_t,a_t)
        &\geq r(s_t,a_t) + \bar{P}_{t}V_t + b_t \\
        &= r(s_t,a_t) + \bar{P}_{t}V_t + \max\left\{c_1\sqrt{\frac{\mathbb{V}(\bar{P}_{t}, V_t)\iota_t}{n_{t}}}, \frac{c_2B_{t+1}\iota_t}{n_{t}}\right\} \\
        &\geq r(s,a) + \bar{P}_{t}V^\star_{[B_{t+1}]} + \max\left\{c_1\sqrt{\frac{\mathbb{V}(\bar{P}_{t}, V^\star_{[B_{t+1}]})\iota_t}{n_{t}}}, \frac{c_2B_{t+1}\iota_t}{n_{t}}\right\} \\
        &\geq r(s_t,a_t) + P_{s_t,a_t}V^\star_{[B_{t+1}]}  \\
        &\geq r(s_t,a_t) + P_{s_t,a_t}V^\star  \\
        &= Q^\star(s_t,a_t),  
    \end{align*}
    where the second inequality is because $V_t(s)\geq V^\star(s)$ by the induction hypothesis and $V_t(s)\geq -B_t\geq -B_{t+1}$ by the algorithm, which jointly gives $V_t(s)\geq V^\star_{[B_{t+1}]}(s)$; then using the monotone property in \pref{lem: monotone}. This finishes the induction. 
    
    We remark that to make the third inequality above hold for all possible $B=\{1,2,4,8,\ldots\}$ with probability $1-\delta$, we need a union bound over $B$'s. Therefore, in the third inequality above, we actually apply Freedman's inequality for the $B=2^{i}$ case with a probability parameter $\delta_i = \frac{\delta}{2i^2} = \frac{\delta}{2(\log_2 B)^2}$ so that $\sum_{i=1}^\infty\delta_i < \delta$. The additional $2(\log_2 B)^2$ factor in the log term is taken cared in the definition of $\iota_t$.

\end{proof}

\begin{proof}[\pref{thm:ssp_main_result}]
    The proof is similar to that of \pref{thm: main thm for reward}. We require the combination of the bounds in \pref{lem: V*-Q*} and \pref{lem: cost sm of var}. Notice that the proof of \pref{lem: V*-Q*} requires the condition $B\geq \Bstar$, the purpose of which is to ensure optimism $Q_t(\cdot,\cdot)\geq Q^\star(\cdot,\cdot)$. For the SSP case, since optimism is ensured by \pref{lem: MVP for cost} without requiring $B\geq \Bstar$, the conclusion of \pref{lem: V*-Q*} still holds, with the $B$ there replaced by $B_{T+1}$ in \pref{alg: cost case parameter free} (i.e., the maximum $B$ used in \pref{alg: cost case parameter free}). Furthermore, with probability at least $1-O(\delta)$, $B_{T+1}\leq 2\Bstar$ according to the arguments in \pref{sec: SSP}. Therefore, we have 
    \begin{align*}
        \sum_{t=1}^{T} \left(V^\star(s_t) - Q^\star(s_t,a_t) \right) 
        &\leq  O\left( \sqrt{SA\sum_{t=1}^{T} \mathbb{V}(P_t, V^\star)\tilde{\iota}_{T,\Bstar,\delta}} + \Bstar S^2 A\tilde{\iota}_{T,\Bstar,\delta}\right).  
    \end{align*}
    The bound in \pref{lem: cost sm of var} can be directly used. Combining them in a similar way as in the proof of \pref{thm: main thm for reward}, we get
    \begin{align}
        \Reg_K = KV^\star(s_\init) - \sum_{t=1}^T r(s_t,a_t)\leq O\left(\sqrt{\Bstar SA\sum_{t=1}^T |r(s_t,a_t)|\tiliota_{T,\Bstar,\delta}} + \Bstar S^2A \tiliota_{T,\Bstar,\delta} \right).  \label{eq: re tmp for SSP}
    \end{align} 
    Recall that in SSP, $V^\star(\cdot)\leq 0$ and $r(\cdot,\cdot)\leq 0$, so \pref{eq: re tmp for SSP} is equivalent to 
    \begin{align*}
        \sum_{t=1}^T |r(s_t,a_t)| - K\Vinit \leq O\left(\sqrt{\Bstar SA\sum_{t=1}^T |r(s_t,a_t)|\tiliota_{T,\Bstar,\delta}} + \Bstar S^2A \tiliota_{T,\Bstar,\delta} \right).
    \end{align*}
    Solving for $\sum_{t=1}^T |r(s_t,a_t)|$, we get 
    \begin{align*}
        \sum_{t=1}^T |r(s_t,a_t)| \leq O\left(K\Vinit +\Bstar S^2A \tiliota_{T,\Bstar,\delta} \right). 
    \end{align*}
    Plugging this back to \pref{eq: re tmp for SSP} finishes the proof. 
\end{proof}

%% file: appendix-lower-bound.tex
\section{Lower Bound for General Stochastic Path}\label{app: general lower}
\begin{proof}[\pref{thm: lower bound for general case}]\ \ 
As mentioned in \pref{fn: conversion}, we can convert between the cases of ``deterministic initial state'' and ``random initial state'' by just introducing one additional state. In our lower bound proof, our construction is based on random initial states, but converting it to deterministic initial state is straightforward. In this proof, we use $P(\cdot|s,a)=P_{s,a}(
\cdot)$ to denote the transition probability. 

We first consider a MDP with two non-terminal states $x, y$ and a terminal state $g$. The number of actions is $A$. Let the initial state distribution be $\text{uniform}\{x,y\}$. For all action $a$, let $r(x,a)=1$ and $r(y,a)=-1$. Let $\epsilon,\Delta\in(0,\frac{1}{4}]$ be quantities to be determined later. For all action $a$, let 
\begin{align*}
    P(\cdot~|~y, a) &=  (1-\epsilon)\text{uniform}\{x, y\} + \epsilon \mathbf{1}_g. 
\end{align*}
For all but one single \emph{good} action $a^\star$, let 
\begin{align*}
    P(\cdot~|~x, a) = (1-\epsilon)\text{uniform}\{x, y\} + \epsilon \mathbf{1}_g.  
\end{align*}
For the good action $a^\star$, let 
\begin{align*}
    P(\cdot~|~x, a^\star) = \frac{1-\epsilon+\Delta}{2}\mathbf{1}_{x} + \frac{1-\epsilon-\Delta}{2}\mathbf{1}_{y} + \epsilon \mathbf{1}_g. 
\end{align*}
First, we calculate the optimal value function of this MDP. \\[0.4cm] 
\textbf{Claim 1}\ \ \ $V^\star(x) = 1+\frac{\Delta}{\epsilon}\frac{1+\epsilon}{2-\Delta}$, and $Q^\star(x,a) = 1+\frac{\Delta}{\epsilon}\frac{1-\epsilon}{2-\Delta}$ for $a\neq a^\star$.  \\
\begin{proof}[Claim 1] 
We first show that the optimal policy is to always choose $a^\star$ on state $x$. We only need to compare two deterministic policies: always choose $a^\star$, or always choose some other action $a\neq a^\star$. For the policy with $\pi(x)=a^\star$, we have 
\begin{align*}
    V^\pi(x) &= r(x,a^\star) + P(x|x,a^\star)V^\pi(x) + P(y|x,a^\star)V^\pi(y) = 1 + \frac{1-\epsilon+\Delta}{2}V^\pi(x) + \frac{1-\epsilon-\Delta}{2} V^\pi(y), \\
    V^\pi(y) &= r(y,\cdot) + P(x|y,\cdot)V^\pi(x) + P(y|y,\cdot)V^\pi(y) = -1 + \frac{1-\epsilon}{2}V^\pi(x) + \frac{1-\epsilon}{2}V^\pi(y). 
\end{align*}
Solving the equations, we get $V^\pi(x) = 1+ \frac{\Delta}{\epsilon}\frac{1+\epsilon}{2-\Delta}$. On the other hand, if $\pi(x)\neq a$, then by similar calculation, we get $V^\pi(x)=1$, which is smaller. This shows that the optimal policy is to always choose $a^\star$ on $x$. Thus, $V^\star(x)=Q^\star(x,a^\star) = 1+\frac{\Delta}{\epsilon}\frac{1+\epsilon}{2-\Delta}$. Plugging this in another bellman equation
\begin{align*}
    V^\star(y) = -1 + \frac{1-\epsilon}{2}V^\star(x) + \frac{1-\epsilon}{2}V^\star(y)
\end{align*}
we get $V^\star(y) = -1 + \frac{\Delta}{\epsilon}\frac{1-\epsilon}{2-\Delta}$, and thus for $a\neq a^\star$, 
\begin{align*}
    Q^\star(x,a) = 1 + \frac{1-\epsilon}{2}\left(V^\star(x) + V^\star(y)\right) = 1 + \frac{1-\epsilon}{2}\left(\frac{\Delta}{\epsilon}\frac{2}{2-\Delta}\right) = 1+\frac{\Delta}{\epsilon}\frac{1-\epsilon}{2-\Delta}. 
\end{align*}
\end{proof}
Next, we follow the proof idea in \cite{rosenberg2020near} and consider a \emph{truncated process}: First, we view the $K$ episodes as a continuous process in which once the learner reaches $g$, a new state is drawn from the initial distribution and the learner restarts from there. Then we cap the process to make it contain at most $\frac{K}{\epsilon}$ step: if the learner has not finished all $K$ episodes after $\frac{K}{\epsilon}$ steps, then we stop the learner before it finishes all $K$ episodes. 

In the original process, we let $\text{Reg}$ to be the regret, $T_x$ and $T_y$ be the number of steps the learner visits $x$ and $y$, respectively, $T_{x,a}$ be the number of steps the learner visits $x$ and chooses action $a$, and let $T=T_x+T_y$. We define $T_x^-, T_y^-, T_{x,a}^-, T^-$ to be the corresponding quantities in the truncated process. We first show the following claim: \\[0.4cm]
\textbf{Claim 2\ \ } $\E[\Reg] \geq \frac{2\Delta}{2-\Delta} \E[T_x^- - T_{x,a^\star}^-]$. 
\begin{proof}[Claim 2]
    We first focus on episode $k$ in the original process. Let the episode starts from $t=t_k$, and the last step in the episode before reaching $g$ is $t=e_{k}$. Then the expected regret is given by
    \begin{align*}
        \E\left[V^\star(s_{t_k}) - \sum_{t=t_k}^{e_{k}} r(s_t,a_t)\right]  
        &= \E\left[V^\star(s_{t_k}) - \sum_{t=t_k}^{e_k}(Q^\star(s_t,a_t) - P_{s_t,a_t}V^\star)\right] \\
        &= \E\left[V^\star(s_{t_k}) - \sum_{t=t_k}^{e_k}(Q^\star(s_t,a_t) - V^\star(s_{t+1}))\right] \\
        &= \E\left[\sum_{t=t_k}^{e_k}(V^\star(s_{t}) - Q^\star(s_t,a_t))\right].
    \end{align*}
    Notice that $V^\star(s_t) - Q^\star(s_t,a_t)=0$ when $s_t=y$ or $(s_t,a_t)=(x,a^\star)$. When $s_t=x$ and $a_t\neq a^\star$, we have $V^\star(s_t) -Q^\star(s_t,a_t) = \frac{\Delta}{\epsilon}\frac{2\epsilon}{2-\Delta} = \frac{2\Delta}{2-\Delta}$ according to Claim 1. Thus, the expected regret in episode $k$ is 
    \begin{align*}
        \E[\Reg_k] = \E\left[\frac{2\Delta}{2-\Delta}\sum_{t=t_k}^{e_k}\one[s_t=x, a_t\neq a^\star]\right] = \E\left[\frac{2\Delta}{2-\Delta}\sum_{t=t_k}^{e_k}(\one[s_t=x] - \one[s_t=x, a_t=a^\star])\right]. 
    \end{align*}
    Summing this over episodes, we get 
    \begin{align*}
        \E[\Reg] = \frac{2\Delta}{2-\Delta}\E[T_x - T_{x,a^\star}]. 
    \end{align*}
    Using the simple fact that $T_x - T_{x,a^\star}\geq T_x^- - T_{x,a^\star}^-$ finishes the proof.
\end{proof}
\ \\
\textbf{Claim 3\ \ } $\E[T_x^-]\geq \frac{K}{6\epsilon}$. 
\begin{proof}[Claim 3] \ \ Since on every state-action pair $(s,a)$, we have $P(x|s,a)\geq P(y|x,a)$ and the initial distribution is uniform between $x$ and $y$, it holds that $\E[T_x^-]\geq \E[T_y^-]$. Thus it suffices to show that $\E[T^-]\geq \frac{K}{3\epsilon}$. Note that $\E[T^- ]= \E[\min\{T, \frac{K}{\epsilon}\}] \geq \sum_{k=1}^K \E[\min\{L_k, \frac{1}{\epsilon}\}]$ where $L_k$ is the length of episode $k$. If we can show that $L_k\geq \frac{1}{\epsilon}$ with probability at least $\frac{1}{3}$, then we have $\sum_{k=1}^K \E[\min\{L_k, \frac{1}{\epsilon}\}] \geq \frac{1}{3}\sum_{k=1}^K \min\{\frac{1}{\epsilon}, \frac{1}{\epsilon}\} = \frac{K}{3\epsilon}$. By the definition of the transition kernel, we indeed have $\Pr[L_k\geq \frac{1}{\epsilon}]\geq (1-\epsilon)^{\frac{1}{\epsilon}-1}\geq \frac{1}{e}\geq \frac{1}{3}$. This finishes the proof. 
\end{proof}

The following proof follows the standard lower bound proof for multi-armed bandits \citep{auer2002nonstochastic}. We create $A$ different instances of MDPs, where in each of them, the choice of the good action is different. We use $\mathbb{P}_a$ and $\E_a$ to denote the probability measure and the expectation under the instance where $a$ is chosen as the good action. We further introduce an instance where every action behaves the same and there is no good action. We use $\mathbb{P}_\unif$ and $\E_\unif$ to denote the probability and expectation under this instance.
\\[0.4cm]
\textbf{Claim 4}\ \ $\E_a[T_{x,a}^-] \leq \E_\unif[T_{x,a}^-] + O\left(\frac{K\Delta}{\epsilon} \sqrt{\E_{\unif}[T^-_{x,a}]}\right)$.   
\begin{proof}[Claim 4]\ \ 
   With standard arguments from \cite{auer2002nonstochastic}, because $0\leq T_{x,a}^-\leq \frac{K}{\epsilon}$, we have 
   \begin{align*}
       &\E_a[T_{x,a}^-] - \E_{\unif}[T_{x,a}^-] \\ &\leq \frac{K}{\epsilon} \|\mathbb{P}_a - \mathbb{P}_{\unif}\|_1 \\
       &\leq \frac{K}{\epsilon}\sqrt{\frac{1}{2}\text{KL}(\mathbb{P}_a, \mathbb{P}_\unif)}  \\
       &= \frac{K}{\epsilon}\sqrt{\frac{1}{2}\sum_{t=1}^{K/\epsilon-1} \mathbb{P}_\unif[(s_{t}, a_{t})=(x,a)] \KL\left(\text{Multi}\left(\frac{1-\epsilon}{2}, \frac{1-\epsilon}{2}, \epsilon\right), \text{Multi}\left(\frac{1-\epsilon+\Delta}{2}, \frac{1-\epsilon-\Delta}{2}, \epsilon\right)\right)} \tag{by the chain rule of \text{KL}; we use $\text{Multi}(a,b,c,\ldots)$ to denote a multinomial distribution}\\
       &\leq O\left(\frac{K}{\epsilon}\sqrt{\sum_{t=1}^{K/\epsilon}\mathbb{P}_\unif[(s_t,a_t)=(x,a)] \times \frac{\Delta^2}{1-\epsilon}} \right) \\
       &= O\left(\frac{K\Delta}{\epsilon} \sqrt{\E_{\unif}[T^-_{x,a}]}\right). 
   \end{align*}
\end{proof}
Thus, the expected regret is 
\begin{align*}
    \frac{1}{A} \sum_{a} \E[\Reg] 
    &\geq  \frac{\Delta}{A}\sum_{a} \E_a[T^-_{x} - T^-_{x,a}] \\
    &\geq \frac{ K\Delta}{6\epsilon} - \frac{\Delta}{A}\sum_{a} \E_a[T_{x,a}^-] \\
    &\geq \frac{ K\Delta}{6\epsilon} - \frac{\Delta}{A}\sum_a \left(\E_{\unif}[T^-_{x,a}] + O\left(\frac{K\Delta}{\epsilon}\sqrt{\E_{\unif}[T^-_{x,a}]}\right)\right) \\
    &\geq \frac{ K\Delta}{6\epsilon} - \frac{K\Delta}{A\epsilon} - O\left(\frac{\Delta}{A}\times \frac{K\Delta}{\epsilon}\times \sqrt{A\sum_a \E_{\unif}[T^-_{x,a}]}\right) \\
    &\geq \frac{ K\Delta}{6\epsilon} - \frac{K\Delta}{A\epsilon} - O\left(\frac{K\Delta^2}{\epsilon}\sqrt{\frac{K}{A\epsilon}}\right) \\
    &= \frac{K\Delta}{\epsilon}\left(\frac{ 1}{6} - \frac{1}{A} - O\left(\Delta\sqrt{\frac{K}{A\epsilon}}\right)\right). 
\end{align*}
Picking $\Delta=\Theta\left(\sqrt{\frac{A\epsilon}{K}}\right)$, we get that $\E[\Reg]\geq \Omega\left(\sqrt{\frac{AK}{\epsilon}}\right)$. 

Notice that for the MDP we constructed, we have 
\begin{align*}
    R^2 &= \sup_{\pi}\E^{\pi}\left[\left(\sum_{t=1}^\tau r(s_t,a_t)\right)^2 \right] \\
    &= \sup_{\pi}\E^{\pi}\left[\left(\sum_{t=1}^\tau (r(s_t,a_t) - \E_{t}[r(s_t,a_t)])\right)^2 + \left(\sum_{t=1}^\tau \E_{t}[r(s_t,a_t)] \right)^2 \right]  \\
    &\leq \sup_{\pi}\E^{\pi}\left[\tau + (\Delta \tau)^2\right] \\
    & \leq \frac{1}{\epsilon} + \frac{2\Delta^2}{\epsilon^2}.  
\end{align*}
We let $\epsilon=\frac{1}{u^2}$. With this choice, $R^2\leq \frac{1}{\epsilon} + \frac{2\Delta^2}{\epsilon^2} \leq O\left(\frac{1}{\epsilon} + \frac{A}{\epsilon K}\right) \leq O\left(\frac{1}{\epsilon}\right)=O(u^2)$ if we assume $K\geq A$. Therefore, the condition in the lemma the satisfied, and the regret lower bound given by $\Omega\left(\sqrt{\frac{AK}{\epsilon}}\right)=\Omega(u\sqrt{AK})$. 

To show the lower bound for general $S$, we make $\frac{S}{2}$ copies of the MDP we constructed, and let the initial state distribution be uniform over all states. When we run on this aggregated MDP for $K$ episodes, a constant portion of the $\frac{S}{2}$ copies will be visited for $\Theta(\frac{K}{S})$ times. Using the lower bound we just established above (with $K$ replaced by $\frac{K}{S}$), we have that the regret is lower bounded by
\begin{align*}
    \Omega\left(S\times u\sqrt{ \frac{AK}{S}}\right) = \Omega(u\sqrt{SAK}). 
\end{align*}
Notice that the assumption we need becomes $\frac{K}{S}\geq A$, or $K\geq SA$. 

\end{proof}

\begin{proof}[\pref{thm: impossibility of R*}]
    We first prove this theorem for an MDP with two non-terminal state $x, y$ and a terminal state $g$. The number of actions is $A$. The construction is similar to that in \pref{thm: lower bound for general case}. On state $x$, there is potentially a good action $a^\star\in[A-1]$. The reward function and transition kernel are chosen as below: 
    \begin{center}
    \begin{tabular}{c|c||c|c|c|c|c}
        state & action & reward & $\rightarrow z$ & $\rightarrow x$ & $\rightarrow y$ & $\rightarrow g$  \\
        \hline
        $x$ & $[A-1]\backslash\{a^\star\}$ & $1$ & 0 & $\frac{1-\epsilon-\Delta}{2}$ & $\frac{1-\epsilon+\Delta}{2}$ &
        $\epsilon$ \\
        \hline
        $x$ & $a^\star$ & $1$ & 0 & $\frac{1-\epsilon+\Delta}{2}$ & $\frac{1-\epsilon-\Delta}{2}$ &
        $\epsilon$ \\
        \hline
        $x$ & $A$ & $1$ & $0$ & $0$ & $0$ & $1$ \\
        \hline
        $y$ & $[A]$ & $-1$ & $0$ & $\frac{1-\epsilon}{2}$ & $\frac{1-\epsilon}{2}$ & $\epsilon$ 
    \end{tabular}
    \end{center}
    The choices of $\epsilon, \Delta$ satisfy $\epsilon,\Delta\in(0, \frac{1}{8}]$ and $\Delta^2\leq \epsilon$. 
    
    We consider two cases: one with the good action $a^\star$, and the other without the good action. We first find the optimal policy in each case. If there is a good action on $x$, then for the policy that always choose $a^\star$ on $x$, we have 
    $V^\pi(x) = 1 + \frac{1-\epsilon+\Delta}{2}V^\pi(x) + \frac{1-\epsilon-\Delta}{2}V^\pi(y)$ and $V^\pi(y) = -1 + \frac{1-\epsilon}{2}V^\pi(x) + \frac{1-\epsilon}{2}V^\pi(y)$, which jointly give $V^\pi(x)=1+\frac{\Delta}{\epsilon}\frac{1+\epsilon}{2-\Delta}$; for the policy that always choose $[A-1]\backslash \{a^\star\}$ on $x$, we have (by similar calculation) $V^\pi(x) = 1-\frac{\Delta}{\epsilon}\frac{1+\epsilon}{2+\Delta}$; for the policy that always choose $A$ on $x$, we have $V^\pi(x)=1$. Clearly, the one that always choose $a^\star$ on $x$ gives the highest expected total reward, so it is the optimal policy in this case. Therefore, when there is a good action, 
    \begin{align}
        V^\star(x) &= Q^\star(x,a^\star) = 1 +\frac{\Delta}{\epsilon}\frac{1+\epsilon}{2-\Delta}, \\
        V^\star(y) &= -1 + \frac{1-\epsilon}{2}V^\star(x) + \frac{1-\epsilon}{2}V^\star(y) = -1 + \frac{\Delta}{\epsilon}\frac{1-\epsilon}{2-\Delta}, \\
        V^\star(x) - Q^\star(x,a) &= Q^\star(x,a^\star) - Q^\star(x,a), \nonumber \\ 
        & =  \Delta\left(V^\star(x) - V^\star(y)\right) 
        = \Delta\left(2 + \frac{2\Delta}{2-\Delta}\right)\geq 2\Delta,\ \ \ \  \text{for $a\neq a^\star, A$}   \label{eq: suboptimality for a neq a* A}\\ 
        V^\star(x) - Q^\star(x,A) &= \left(1+\frac{\Delta}{\epsilon}\frac{1+\epsilon}{2-\Delta}\right) - 1 \geq \frac{\Delta}{2\epsilon}. \label{eq: suboptimality for A}
    \end{align} 
    If there is no good action, then the optimal policy is to always choose action $A$ on state $x$. In this case, we have 
    \begin{align}
        V^\star(x) &= Q^\star(x,a^\star) = 1, \\
        V^\star(y) &= -1 + \frac{1-\epsilon}{2}V^\star(x) + \frac{1-\epsilon}{2}V^\star(y) = -1,  \\
        V^\star(x) - Q^\star(x,a) &= 1 - \left[1 + \frac{1-\epsilon-\Delta}{2}V^\star(x) + \frac{1-\epsilon+\Delta}{2}V^\star(y)\right] = \Delta\ \ \ \text{for $a\neq A$}  \label{eq: regret for x state suboptimal a}
    \end{align}
    
    We use $\mathbb{P}_{a}$ and $\E_{a}$ to denote the probability measure and expectation under the environment where $a\in[A-1]$ is chosen as the good action on $x$; we use $\mathbb{P}_{\unif}$ and $\E_{\unif}$ to denote the probability and expectation under the environment where there is no good action. 
    
    For both kinds of MDPs, we have 
    \begin{align}
        R^2 = \Rmax^2 &= \sup_{\pi}\E^{\pi}\left[\left(\sum_{t=1}^\tau r(s_t,a_t)\right)^2 \right] \tag{$\tau$ is the last step before reaching $g$ } \nonumber \\
        &= \Theta\left(\sup_{\pi}\E^{\pi}\left[\left(\sum_{t=1}^\tau (r(s_t,a_t) - \E_{t}[r(s_t,a_t)])\right)^2 + \left(\sum_{t=1}^\tau \E_{t}[r(s_t,a_t)] \right)^2 \right] \right) \nonumber \\
        &=\Theta\left( \sup_{\pi}\E^{\pi}\left[\tau + (\Delta \tau)^2\right]\right) \nonumber \\
        & =\Theta \left(\frac{1}{\epsilon} + \frac{\Delta^2}{\epsilon^2}\right) = \Theta\left(\frac{1}{\epsilon}\right),    \label{eq: R calculation} 
    \end{align}
    where in the last equation we use the assumption $\Delta^2\leq \epsilon$. 
    For the MDP with good action, we have $R_\star = R$ since the optimal policy always choose $a^\star$ on $x^\star$;
    for the MDP without good action, $R_\star=\Theta(1)$ since the optimal policy will take action $A$ on all state $x$, and directly go to $g$. 
    
    Like in the proof of \pref{thm: lower bound for general case}, we consider the truncated process where the total number of steps is truncated to $\frac{K}{\epsilon}$ (and ignoring the regret incurred later). In the truncated process, we let $T_{x}$ to denote the number of times the learner visits $x$, and $T_{x,a}$ denote the number of times the learner visits $x$ and choose action $a$. 

    With all the calculations above, we have the following: 
    \begin{align}
        \E_{\unif}[\Reg_K] 
        &\geq \E\left[\sum_{t=1}^{\frac{K}{\epsilon}} \left(V^\star(s_t) - Q^\star(s_t,a_t)\right)\right] \geq \E\left[(T_x - T_{x,A})\Delta \right],  \label{eq: regret in the unif case}
    \end{align} 
    where we use \pref{eq: regret for x state suboptimal a}. On the other hand, 
    \begin{align}
        \E_{a}[\Reg_K] \geq \E_a\left[T_{x,A}\times \frac{\Delta}{2\epsilon} + (T_x-T_{x,a}-T_{x,A})\times 2\Delta  \right] \geq \E_a\left[T_{x,A}\times \frac{\Delta}{4\epsilon} + (T_x-T_{x,a})\times 2\Delta  \right], \label{eq: Ea RegK}
    \end{align}
    where we use \pref{eq: suboptimality for a neq a* A} and \pref{eq: suboptimality for A} and that $\epsilon\leq \frac{1}{8}$.  
    
    By the same arguments as in Claim 4 in the proof of \pref{thm: lower bound for general case}, we have that for $a\in[A-1]$, 
    \begin{align}
        \big|\E_{a}[T_x- T_{x,a}] - \E_{\unif}[T_x - T_{x,a}]\big| \leq  O\left(\frac{K\Delta}{\epsilon}\sqrt{\E_{\unif}[T_{x,a}]}\right)     \label{eq: same argument 1}
    \end{align}
    and 
    \begin{align}
        \big|\E_{\unif}[T_{x,A}] - \E_{a}[T_{x,A}]\big| \leq O\left(K\Delta\sqrt{\E_{\unif}[T_{x,a}]}\right)  \label{eq: same argument 2}
    \end{align}
    where we use that $T_x - T_{x,a}\in[0,\frac{K}{\epsilon}]$ and $T_{x,A}\in[0,K]$. 
    In an environment where $a^\star$ is chosen randomly from $[A-1]$, the expected regret is 
    \begin{align*}
        &\frac{1}{A-1}\sum_{a=1}^{A-1}\E_a[\Reg_K] \\
        &\geq \frac{1}{A-1}\sum_{a=1}^{A-1} \E_a\left[\frac{T_{x,A}}{4\epsilon}\Delta + 2(T_x - T_{x,a})\Delta\right]   \tag{by \pref{eq: Ea RegK}}\\ 
        &= \frac{\Delta}{4\epsilon}\frac{1}{A-1}\sum_{a=1}^{A-1}\E_{a}[T_{x,A}] + 2\Delta\times \frac{1}{A-1}\sum_{a=1}^{A-1}\E_a[T_x-T_{x,a}] \\
        &\geq \frac{\Delta}{4\epsilon}\frac{1}{A-1}\sum_{a=1}^{A-1} \left(\E_{\unif}[T_{x,A}] - O\left(K\Delta\sqrt{\E_{\unif}[T_{x,a}]}\right)\right) \\
        &\qquad \qquad + 2\Delta \times \frac{1}{A-1}\sum_{a=1}^{A-1}\left(\E_{\unif}\left[T_{x}-T_{x,a} \right]- O\left(\frac{K\Delta}{\epsilon}\sqrt{\E_{\unif}[T_{x,a}]}\right)\right) \tag{by \pref{eq: same argument 1} and \pref{eq: same argument 2}}\\
        &\geq \frac{\Delta}{4\epsilon}\E_{\unif}[T_{x,A}]  + 2\Delta \E_{\unif}[T_x] - \frac{2\Delta}{A-1}\E_{\unif}[T_x-T_{x,A}] - O\left(\frac{K\Delta^2}{\epsilon}\sqrt{\frac{1}{A-1}\sum_{a=1}^{A-1}\E_{\unif}[T_{x,a}]}\right)  \tag{AM-GM} \\
        &\geq \frac{\Delta}{4\epsilon}\E_{\unif}[T_{x,A}]  + \Delta \E_{\unif}[T_x] - O\left(\frac{K\Delta^2}{\epsilon}\sqrt{\frac{1}{A-1}\sum_{a=1}^{A-1}\E_{\unif}[T_{x,a}]}\right). 
    \end{align*}
    Before continuing, we prove a property: 
    \paragraph{Claim 1}\ \ $\E[T_x]\geq \frac{K-2\E[T_{x,A}]}{24\epsilon}$. 
    \begin{proof}[Claim 1]
        We focus on a single episode. Let $N_x$ be the total number of steps the learner visits $x$ in an episode, $N_{x,A}$ be the number of steps the learner visits $x$ and chooses $A$. Clearly, $N_{x,A}\leq 1$ since once the learner chooses $A$, the episodes ends. 
        
        We prove the following statement: if $\E[N_{x,A}]\leq \frac{1}{2}$, then $\E[N_x]\geq \frac{1}{24\epsilon}$. This can be seen by the following: let $\calE_{A}$ be the event that in the episode, the learner \emph{ever} chooses action $A$ when visiting $x$, and let $\calE_A'$ be its complement event (i.e., replacing \emph{ever} by \emph{never}). Then
        \begin{align*}
            \Pr\left[N_x\geq \frac{1}{3\epsilon}\right] \geq \Pr\left[\calE_{A}'\right]\times\Pr\left[N_x\geq \frac{1}{3\epsilon}~\Big|~ \calE_{A}'\right]  
        \end{align*}
        Notice that $\Pr[\calE_A] = \E[N_{x,A}]\leq \frac{1}{2}$, so $\Pr[\calE_{A}']\geq \frac{1}{2}$. Then notice that 
        \begin{align*}
            \Pr\left[N_x\geq \frac{1}{3\epsilon}~\Big|~ \calE_{A}'\right] \geq \left(1-3\epsilon\right)^{\frac{1}{3\epsilon}}\geq \left(1-\frac{1}{2}\right)^{2} = \frac{1}{4}
        \end{align*}
        because every time the learner select an action in $[A-1]$, with probability at least
        \begin{align*}
            \frac{1-\epsilon-\Delta}{2} + \frac{1-\epsilon+\Delta}{2}\left(\frac{1-\epsilon}{2} + \left(\frac{1-\epsilon}{2}\right)^2 + \cdots\right) = \frac{1-\epsilon-\epsilon\Delta}{1+\epsilon}\geq 1-3\epsilon 
        \end{align*}
        he will visit state $x$ again. Hence, $\E[N_x]\geq \frac{1}{3\epsilon}\Pr\left[N_x\geq \frac{1}{3\epsilon}\right]\geq \frac{1}{3\epsilon}\times \frac{1}{2}\times \frac{1}{4}=\frac{1}{24\epsilon}$. 
        
        Now we consider all $K$ episodes, and denote $N_x^{(k)}$, $N_{x,a}^{(k)}$ denotes the visitation counts that correspond to episode $k$. By the discussion above, we have 
        \begin{align*}
            \E\left[T_x\right] &= \sum_{k=1}^K \E\left[N_x^{(k)}\right] \geq \sum_{k=1}^K \frac{1}{24\epsilon}\ind{\E\left[N_{x,A}^{(k)}\right] \leq \frac{1}{2}}  \\
            &= \frac{K}{24\epsilon} - \sum_{k=1}^K \frac{1}{24\epsilon}\ind{\E\left[N_{x,A}^{(k)}\right] > \frac{1}{2}} \\
            &\geq \frac{K}{24\epsilon} - \sum_{k=1}^K \frac{2}{24\epsilon}\E\left[N_{x,A}^{(k)}\right] \\
            &= \frac{K}{24\epsilon} - \frac{1}{12\epsilon}\E[T_{x,A}]. 
        \end{align*}
        
    \end{proof}
    
    With Claim 1, we continue with the previous calculation: 
    \begin{align*}
        \frac{1}{A-1}\sum_{a=1}^{A-1}\E_{a}[\Reg_K] 
        &\geq \frac{\Delta}{4\epsilon}\E_{\unif}[T_{x,A}]  + \Delta \E_{\unif}[T_x] - O\left(\frac{K\Delta^2}{\epsilon}\sqrt{\frac{1}{A-1}\sum_{a=1}^{A-1}\E_{\unif}[T_{x,a}]}\right) \\
        &\geq \frac{\Delta}{12\epsilon}\E_{\unif}[T_{x,A}] + \Delta\E_{\unif}[T_x] - O\left(\frac{K\Delta^2}{\epsilon}\sqrt{\frac{1}{A}\E_{\unif}[T_x-T_{x,A}]}\right) \\
        &\geq \frac{K\Delta}{24\epsilon} - O\left(\frac{K\Delta^2}{\epsilon}\sqrt{\frac{1}{A}\E_{\unif}[T_x-T_{x,A}]}\right) 
    \end{align*}
    
    Suppose that the algorithm can ensure $\E[\Reg_K]\leq O(u\sqrt{AK})$ for all instances with $R_{\max}\leq u$. By the bound in \pref{eq: R calculation}, there exists universal constants $c_2>0$ and a term $c_1$ that only involves logarithmic factors such that 
    \begin{align*}
        \frac{K\Delta}{24\epsilon} - \frac{c_2 K\Delta^2}{\epsilon}\sqrt{\frac{1}{A}\E_{\unif}[T_x-T_{x,A}]}\leq c_1\sqrt{\frac{1}{\epsilon} AK}
    \end{align*}
    We pick $\Delta=48c_1\sqrt{\frac{\epsilon A}{ K}}$ (one can verify that this satisfies $\Delta^2\leq \epsilon$ as long as $K\geq \Omega(c_1^2 A)$). Then the inequality above reads
    \begin{align*}
        2c_1\sqrt{\frac{1}{\epsilon}AK} -  48^2 c_1^2 c_2\sqrt{A\E_{\unif}[T_x-T_{x,A}]} \leq c_1\sqrt{\frac{1}{\epsilon}AK}, 
    \end{align*}
    which is equivalent to 
    \begin{align*}
        \E_{\unif}[T_x-T_{x,A}]\geq \frac{K}{48^4c_1^2c_2^2\epsilon}. 
    \end{align*}
    This, together with \pref{eq: regret in the unif case}, implies 
    \begin{align}
        \E_{\unif}[\Reg_K] \geq \frac{K\Delta}{48^4c_1^2c_2^2\epsilon} = \frac{1}{48^3c_1c_2^2}\sqrt{\frac{AK}{\epsilon}}. \label{eq: bound for uniform env} 
    \end{align}
    Recall that $\unif$ specifies the environment where there is no good action, and in this case $R_\star=\Theta(1)$. However, the bound \pref{eq: bound for uniform env} scales with $R=\Rmax=\Theta(\sqrt{1/\epsilon})\gg R_\star$. 
    
    
    
    Now we generalize our construction to general number of states $S$. We construct an MDP with $S$ non-terminating states that consists of $\frac{S}{2}$ copies of the two-state MDP we just constructed, and equip every state $x$ with two additional actions. To connect these $\frac{S}{2}$ copies, we create a balanced binary tree with $\frac{S}{2}$ nodes; each node of the tree is the $x$ in the two-state MDP. The two additional actions on every state $x$ will lead to a reward of zero and deterministic transitions to the left and the right child of the node, respectively. Furthermore, we only let at most one of the copies to have the optimal action $a^\star$. The initial state for this tree-structured MDP is its root. 

    Below, we argue that this tree-structured MDP is at least as hard as the original two-state MDP with $\Theta(SA)$ actions (to create this two-state MDP, we let its actions on state $x$ be the union over the actions on all states $x$ in the tree-structured MDPs; similar for $y$). We can see that for any algorithm in the tree-structured MDP, there is a corresponding algorithm for the two-state MDP that achieves the same expected reward. Besides, the expected reward for the optimal policy on the tree-structured MDP and the two-state MDP are the same. Therefore, our lower for general $S$ can be simply obtained through the two-state construction with $SA$ actions. This finishes the proof. 
\end{proof}

\section{Lower Bound for Stochastic Longest Path}\label{app: lower bound SLP}

\renewcommand{\arraystretch}{1.2}

\begin{proof}[\pref{thm: reward setting lower bound for agnostic alg}]
    In this proof, we use $P(\cdot|s,a)=P_{s,a}(\cdot)$ to denote transition probability. We first prove this theorem for $S=2$ (excluding goal state). We assume that the regret bound claimed by the algorithm for the $\Vinit=\Bstar\leq v$ case is
    \begin{align}
        \E[\Reg_K] \leq c v\sqrt{AK} \label{eq: claimed bound}
    \end{align}
    for some $c$ that only involves logarithmic terms. 

    We create an SLP with an two non-terminal states $x,y$, a terminal state $g$, and $A$ actions $\{1, 2, \ldots, A\}$. The initial state is $x$. The reward function is a constant $1$ for all actions on all non-terminating state (i.e., the total reward is the total number of steps before reaching $g$). On state $x$, there is a special action $b$ such that $P(g|x,b)=1-\frac{\sqrt{A}}{c\sqrt{K}\ln^2(vK)}$ and $P(y|x,b) = \frac{\sqrt{A}}{c\sqrt{K}\ln^2(vK)}$; for all other actions $a\in[A]\backslash\{b\}$,  $P(g|x,a)=P(x|x,a)=\frac{1}{2}$. On state $y$, there is potentially a good action $a^\star$ such that $P(g|y,a^\star)=\frac{\sqrt{A}}{2cv\sqrt{K}\log^2K}$ and $P(y|y,a^\star)=1-\frac{\sqrt{A}}{2cv\sqrt{K}\ln^2(vK)}$; for other actions  $a\in[A]\backslash\{a^\star\}$, $P(g|y,a)=P(y|y,a) = \frac{1}{2}$. The transition kernel is summarized in \pref{tab: theorem 9 table}. 
    
    \begin{table} 
    
    \begin{center}
    \begin{tabular}{c|c||c|c|c}
    state & action &  $\rightarrow x$ & $\rightarrow y$ & $\rightarrow g$\\
    \hline
    $x$ & $b$ & $0$ & $\frac{\sqrt{A}}{c\sqrt{K}\ln^2(vK)}$ & $1-\frac{\sqrt{A}}{c\sqrt{K}\ln^2(vK)}$\\
    $x$ & $[A]\backslash \{b\}$  & $1-\frac{1}{v}$ & $0$ & $\frac{1}{v}$\\
    \hline
    $y$ & $a^\star$ & $0$ & $1-\frac{\sqrt{A}}{2cv\sqrt{K}\ln^2(vK)}$ & $\frac{\sqrt{A}}{2cv\sqrt{K}\ln^2(vK)}$\\
    $y$ & $[A]\backslash \{a^\star\}$  & $0$ & $\frac{1}{2}$ & $\frac{1}{2}$
\end{tabular}
\end{center}
\caption{Transition kernel for \pref{thm: reward setting lower bound for agnostic alg} }\label{tab: theorem 9 table}
\end{table}

    Let $\mathbb{P}_a$ and $\mathbb{E}_a$ denote the probability measure and expectation in the instance where $a^\star=a$, and let $\mathbb{P}_{\unif}$ and $\mathbb{E}_{\unif}$ denote those in the instance where $a^\star$ does not exist. 
    
    Let $N_a$ be the total number of times (in $K$ episodes) the learner chooses action $a$ on state $y$, $N_b$ be the total number of times the learner chooses action $b$ on state $x$, and $N_y$ be the total number of times the learner visits state $y$.  
    
    In the environment where $a^\star$ does not exist, the optimal policy on state $x$ is to choose any action in $[A]\backslash\{b\}$. This is because for any policy such that $\pi(x)=b$, we have $V^\pi(x) = 1 + P(y|x,b)V^\pi(y) \leq 3$, and for any policy such that $\pi(x)\in[A]\backslash\{b\}$, we have $V^\pi(x)=v$. Therefore, 
    \begin{align*}
        \E_{\unif}[\Reg_K] = (v-3)\E_{\unif}[N_b] \geq \frac{v}{2}\E_{\unif}[N_b]. 
    \end{align*}
    In this environment, $V^\star(x)=v$ and $V^\star(y)=2$, and thus $\Vinit=\Bstar=v$. By the assumption for the algorithm, we have
    \begin{align*}
        \frac{v}{2}\E_{\unif}[N_b]\leq \E_{\unif}[\Reg_K]\leq cv\sqrt{AK}, 
    \end{align*}
    or $\E_{\unif}[N_b]\leq 2c\sqrt{AK}$. 
    
    On the other hand, in the environment where $a^\star=a$, the optimal policy is to choose action $b$ on state $x$ until transitioning to state $y$ and then choose $a^\star$ on $y$. This is because for this policy,  $V^\pi(x) \geq P(y|x,b)V^\pi(y) = \frac{\sqrt{A}}{c\sqrt{K}\ln^2(vK)}\times \frac{2cv\sqrt{K}\ln^2(vK)}{\sqrt{A}}=2v$, while for all other policies, $V^\pi(x)\leq v$. For this environment, we have $\Vinit=V^\star(x)=Q^\star(x,b)=2v$, and $\Bstar=2cv\ln^2(vK)\sqrt{\frac{K}{A}}$.
    
    
    Next, we bound the KL divergence between the two environments. Note that we have
    \begin{align*}
        \frac{1}{A}\sum_{a=1}^A \E_\unif[N_a] = \frac{1}{A} \E_{\unif}[N_y]=\frac{1}{A}P(y|x,b)\E_{\unif}[N_b]=\frac{1}{c\sqrt{AK}\ln^2(vK)}\E_{\unif}[N_b] \leq \frac{2}{\log^2(vK)}
    \end{align*}
    where in the last inequality we use $\E_{\unif}[N_b]\leq 2c\sqrt{AK}$ which we just derived above. 
    Hence,  
    \begin{align*}
        \frac{1}{A}\sum_{a=1}^A \KL(\mathbb{P}_a,\mathbb{P}_{\unif}) 
        &= \frac{1}{A}\sum_{a=1}^A \mathbb{E}_{\unif}[N_a]\KL\left(\text{Bernoulli}\left(\frac{\sqrt{A}}{2cv\sqrt{K}\ln^2(vK)}\right), \text{Bernoulli}\left(\frac{1}{2}\right)
        \right)\\
        &\leq O\left(\frac{1}{A}\sum_{a=1}^A\E_{\unif}[N_a]\ln(vK)\right) \\ 
        &=O\left(\frac{1}{\ln(vK)}\right)\,.
    \end{align*}
    
    We consider the environment where $a^\star$ is chosen uniformly randomly from $[A]$, and denote its probability measure and expectation as $\overline{\mathbb{P}}$ and $\overline{\E}$. Clearly,  $\overline{\mathbb{P}} = \frac{1}{A}\sum_{a=1}^A \mathbb{P}_a$ and $\overline{\E} = \frac{1}{A}\sum_{a=1}^A \E_a$. By the convexity of KL divergence, we have 
    \begin{align*}
        \KL(\overline{ \mathbb{P}}, \mathbb{P}_{\unif})\leq \frac{1}{A}\sum_{a=1}^A \KL(\mathbb{P}_a, \mathbb{P}_{\unif})=O\left(\frac{1}{\ln(vK)}\right). 
    \end{align*}
    By the same argument as in the proof of Claim 3, \pref{thm: lower bound for general case}, we have
    \begin{align*}
        \overline{\bbP}(N_b > K/2) 
        &\leq \bbP_\unif(N_b > K/2) + O\left(\sqrt{\KL(\overline{\bbP}, \bbP_{\unif})}\right) \\
        &\leq \frac{2}{K}\times \E_{\unif}[N_b] + O\left(\frac{1}{\sqrt{\ln(vK)}}\right) \\
        &\leq O\left(\frac{c\sqrt{A}}{\sqrt{K}} + \frac{1}{\sqrt{\ln(vK)}}\right)
    \end{align*}
    Thus, for large enough $K$, we have 
    $\overline{\bbP}(N_b > K/2) \leq \frac{1}{2}$ and hence the regret in this environment is bounded by $\Omega(v K)$.
    
    To generalize the result to general number of states, we leave the initial state $x$ unchanged, but replace the state $y$ by a binary tree with $S-1$ nodes with root $y_1$ and leaves $y_{S/2}, \ldots, y_{S-1}$. The transition probability $P(y_1|x,b)$ takes the value of $P(y|x,b)$ specified in \pref{tab: theorem 9 table}. For all non-leaf nodes, there are only two actions that can deterministically transition to its two children with zero reward. For all leaf nodes, the transition and reward are same as the $y$ node in \pref{tab: theorem 9 table}. Among all leaves, there is at most one action on one node being the good action $a^\star$. 
    
    Similar to the proof of \pref{thm: impossibility of R*}, it is not hard to see that this MDP is at least as hard as the original two-state MDP with $\Omega(SA)$ actions on $y$, by letting the action set on $y$ in the two-state MDP to be the union of actions on the leaves in the tree-structed MDP. This is because the optimal value on these two cases are the same, and every policy in the tree-structured MDP can find a corresponding policy in the two-state MDP with the same expected reward. Hence our lower bound for general $S$ can be obtained by our original constant $S$ construction with $\Theta(SA)$ actions. 
    

\end{proof}

%% file: appendix-auxiliary.tex
\section{Auxiliary Lemmas}


\begin{lemma}\label{lem: X lemma}
    Let $X_t\in[-c, c]^S$ be in the filtration of $(s_1, a_1, \ldots, s_{t-1}, a_{t-1}, s_t)$ for some $c>0$ with $X_t(g)\triangleq 0$.  Then with probability at least $1-\delta$, 
    \begin{align*}
        \sum_{t=1}^T \mathbb{V}(P_t,X_t) \leq O\left( c\sum_{t=1}^T |X_t(s_t) - P_tX_t| + c\sum_{t=1}^T  \|X_t-X_{t+1}\|_\infty + c^2\ln(1/\delta) \right).
    \end{align*}
\end{lemma}
\begin{proof}
     Denote $P_t:=P_{s_t,a_t}$. 
     \begin{align*}
        &\sum_{t=1}^T \mathbb{V}(P_t, X_t) \\ 
        &= \sum_{t=1}^T \left( P_t X_t^2 -  (P_t X_t)^2 \right)\\
        &= \sum_{t=1}^T \left(P_tX_t^2 - X_t(s_t')^2 \right) + \sum_{t=1}^T \left(X_t(s_t')^2 - X_t(s_t)^2\right) + \sum_{t=1}^T \left(X_t(s_t)^2 - (P_tX_t)^2\right) \\
        &\leq O\left(\sqrt{\sum_{t=1}^T \mathbb{V}(P_t,X_t^2) \ln(1/\delta)} + c^2\ln(1/\delta)\right) + \sum_{t=1}^T (X_t(s_{t+1})^2 - X_t(s_t)^2) + \sum_{t=1}^T \left(X_t(s_t)^2 - (P_tX_t)^2\right)  \tag{because $X_t(g)=0$} \\
        &\leq O\left(c \sqrt{\sum_{t=1}^T  \mathbb{V}(P_t,X_t)\ln(1/\delta)} + c^2\ln(1/\delta)\right) + 2c\sum_{t=1}^T  \|X_t-X_{t+1}\|_\infty + 2c\sum_{t=1}^T |X_t(s_t) - P_tX_t|\\
        &\leq \frac{1}{2}\sum_{t=1}^T  \mathbb{V}(P_t,X_t) + 2c\sum_{t=1}^T  \|X_t-X_{t+1}\|_\infty + 2c\sum_{t=1}^T |X_t(s_t) - P_tX_t|+ O\left( c^2\ln(1/\delta)\right). \tag{AM-GM inequality}
    \end{align*}
    Solving the inequality we get the desired inequality. 
\end{proof}

\begin{lemma}\label{lem: sum of positive}
    Let $X_1, X_2, \ldots, X_\tau\subset [0,b]$ be a sequence with a random stopping time $\tau$, where $X_i$ is in the filtration of  $\mathcal{F}_i=(X_1, \ldots, X_{i-1})$, for some $b\geq 1$. 
    Suppose that for any $i$, $
        \E\left[\sum_{t=i}^\tau X_t ~|~\mathcal{F}_i\right] \leq B.
    $
    Then
    \begin{enumerate}[label=(\alph*)]
        \item with probability at least $1-\delta$, 
    $
        \sum_{t=1}^\tau X_t \leq O((B+b)\ln(1/\delta)), 
    $ 
    \item $
        \E\big[\left(\sum_{t=1}^{\tau}X_t\right)^2 \big]\leq O\left((B+b)\ln_+(\frac{B+b}{c})\E\big[\sum_{t=1}^\tau X_t\big] +  c^2\right)
    $ for any $1\leq c\leq B+b$. 
    \item $
        \E\big[\left(\sum_{t=1}^{\tau}X_t\right)^2 \big]\leq O\left((B+b)^2\right). 
    $
    \end{enumerate}
\end{lemma}
\begin{proof}
    For a sequence $X_1, X_2, \ldots, X_\tau$, define
    \begin{align*}
        \tau_1 = \min\left\{n\leq \tau:~\sum_{t=1}^n X_t \geq 2B\right\}, 
    \end{align*}
    and for $m\geq 2$, define
    \begin{align*}
        \tau_m = \min\left\{n\leq \tau:~\sum_{t=\tau_{m-1}+1}^n X_t \geq 2B\right\}
    \end{align*}
    If such $\tau_m$ does not exist, let $\tau_m=\infty$. Naturally, we define $\tau_0=0$. 
    
    By the condition stated in the lemma and Markov's inequality, we have 
    \begin{align*}
        \Pr\left[\tau_{m+1}<\infty~|~\tau_m<\infty\right] \leq \Pr\left[ \sum_{t=\tau_m+1}^{\tau} X_t \geq 2B\right] \leq \frac{1}{2}. 
    \end{align*}
    Therefore, $\Pr[\tau_m<\infty] \leq 2^{-m}$ and $\Pr[\tau_m=\infty] \geq 1-2^{-m}$. Also, notice that by the definition of $\tau_i$, we have $\sum_{t=\tau_{i-1}+1}^{\tau_i}X_t\leq 2B+b$ for all $i$ (otherwise, $\sum_{t=\tau_{i-1}+1}^{\tau_i-1}X_t> 2B$, contradicting the definition of $\tau_i$). Thus, if $\tau_m=\infty$, then $$\sum_{t=1}^\tau X_t \leq \sum_{i=1}^{m-1}\sum_{t=\tau_{i-1}+1}^{\tau_i}X_t + \sum_{t=\tau_{m-1}+1}^\tau X_t\leq (2B+b)(m-1)+2B\leq (2B+b)m.$$ Combining the arguments above, we have the following: for any $\delta < 0.5$ (letting $m=\lceil \log_2(1/\delta)\rceil$), with probability at least $1-2^{-m}\geq 1-\delta$, 
    \begin{align*}
        \sum_{t=1}^\tau X_t \leq (2B+b)m = (2B+b)\left\lceil\log_2(1/\delta)\right\rceil \leq 8(B+b)\ln(1/\delta). 
    \end{align*}
    This proves (a). 
    Below we bound the second moment: 
    \begin{align*}
        &\E\left[\left(\sum_{t=1}^\tau X_t\right)^2\right] \\ 
        &\leq  \E\left[\left(\sum_{t=1}^\tau X_t\right)^2~\Bigg|~\tau_M=\infty\right]\Pr\left[\tau_M=\infty\right] \\
        &\qquad + \sum_{m=M+1}^\infty  \E\left[\left(\sum_{t=1}^\tau X_t\right)^2~\Bigg|~\tau_m=\infty\right] \Pr[\tau_{m-1}<\infty, \tau_{m}=\infty]\\
        &\leq (2B+b)M \E\left[\sum_{t=1}^\tau X_t~\Bigg|~\tau_M=\infty\right]  + \sum_{m=M+1}^\infty \left((2B+b)m\right)^2 \Pr[\tau_{m-1}<\infty] \\
        &\leq 2M(B+b)\E\left[\sum_{t=1}^\tau X_t\right] + (2B+b)^2\sum_{m=M+1}^\infty 2^{-m+1}m^2.  \\
    \end{align*}
    If we pick $M = \lceil 4\log_2(\frac{2B+b}{c}) \rceil$, then $(2B+b)^2 2^{-\frac{M}{2}}\leq c^2$, and the last expression can be further upper bounded by 
    \begin{align*}
        & O\left((B+b)\ln_+\left(\frac{B+b}{c}\right)\E\left[\sum_{t=1}^\tau X_t\right] +  c^2\sum_{m=M+1}^\infty 2^{-\frac{m}{2}+1}m^2\right)\\
        &=O\left((B+b)\ln_+\left(\frac{B+b}{c}\right)\E\left[\sum_{t=1}^\tau X_t\right] + c^2\right), 
    \end{align*}
    which proves (b). 
    (c) is an immediate result of (b) by picking $c=B+b$ and noticing that $\E[\sum_{t=1}^\tau X_t]\leq B$.   
\end{proof}

\begin{lemma}[Exercise 5.15 in \cite{lattimore2020bandit}]
Let $X_t$ be a real valued random variable in the filtration of $\calF_t = (X_1,\dots X_{t-1})$ such that $\E[X_t\,|\,\calF_t] = 0$ and assume $\eta>0$, $\E[X_t\,|\,\calF_t]<\eta^{-1}$ a.s. Then with probability at least $1-\delta$ for all $0<t\leq T$, 
$$\sum_{s=1}^tX_s \leq \eta\sum_{s=1}^t\E[X_s^2\,|\,\calF_s]+\eta^{-1}\ln(\delta^{-1})\,.$$
\end{lemma}
\begin{proof}
The claim is stated for a fixed horizon $T$ in \cite{lattimore2020bandit}. However, we can define the surrogate random variable 
\begin{align*}
    X'_i \triangleq  X_i\cdot \left(1-\ind{\exists 0\leq j<i:\, \sum_{s=1}^jX_s > \eta\sum_{s=1}^j\E[X_s^2\,|\,\calF_s]+\eta^{-1}\ln(\delta^{-1})}\right). 
\end{align*}
$X'_i$ is adapted to the filtration and it holds with probability $1-\delta$, 
$$\sum_{s=1}^TX'_s \leq \eta\sum_{s=1}^T\E[X^{\prime 2}_s\,|\,\calF_s]+\eta^{-1}\ln(\delta^{-1})\,,$$
which is equivalent to the anytime result for $X_i$.
\end{proof}

\begin{lemma}
    \label{lem: freedman}
    Let $X_t$ be a real valued random variable in the filtration of $\calF_t = (X_1,\dots X_{t-1})$ such that $\E[X_t\,|\,\calF_t] = 0$ and assume $\E[|X_t|\,|\,\calF_t]<\infty$ a.s. Then with probability at least $1-\delta$ uniformly over all $T>0$, 
\begin{align*}
    \sum_{t=1}^T X_t \leq   4\sqrt{V_T \left(\ln(\delta^{-1})+2\ln\ln(V_T))\right)} + eU_T\ln\left(\frac{2\ln^2(eU_T)}{\delta}\right)\,,
\end{align*}
	where $V_T = \sum_{t=1}^T \E[X_t^2|\calF_t]$, $U_T = \max\{1,\max_{t\in[T]}X_t\}$\,.
\end{lemma}

\begin{proof}
Define $Z_t^{(i)} = X_t \cdot\ind{U_t \leq \exp(i)}$, then $Z_t^{(i)}\exp(-i)\leq 1$ almost surely.
By Exercise 5.15 of \cite{lattimore2020bandit}, with probability at least $1-\delta/(2i^2)$, we have
\begin{align*}
    \sum_{t=1}^T Z_t^{(i)} \leq \sum_{t=1}^T (Z_t^{(i)}-\E[Z_t^{(i)}~|~\calF_t])\leq \exp(-i)\sum_{t=1}^T\E\left[\left(Z_t^{(i)}\right)^2~\Big|~\calF_t\right]+\exp(i)\ln\left(\frac{2 i^2}{\delta}\right)\,.
\end{align*}
By a union bound, this holds with probability $1-\delta$ uniformly over all $i\geq 1$. Note that $\sum_{t=1}^T\E\big[\big(Z_t^{(i)}\big)^2~|~\calF_t\big]\leq \sum_{t=1}^T\E[X_t^{2}~|~\calF_t]=V_T$ and for any $i$ such that $\exp(i) \geq U_T$, we have $\sum_{t=1}^T Z_t^{(i)}=\sum_{t=1}^T X_t$. Hence with probability $1-\delta$, 
\begin{align*}
    \sum_{t=1}^T X_t &\leq \min_{i: \exp(i)\geq U_T} \exp(-i)V_T + \exp(i)\ln\left(\frac{2 i^2}{\delta}\right)\\
    &\leq eU_T\ln\left(\frac{2\ln(eU_T)^2}{\delta}\right)+\min_{i\in \mathbb Z }\exp(-i)V_T + \exp(i)\ln\left(\frac{2 i^2}{\delta}\right)\\
    &\leq eU_T\ln\left(\frac{2\ln(eU_T)^2}{\delta}\right)+2\min_{\gamma > 0 }\gamma^{-1}V_T + \gamma\ln\left(\frac{2 \ln(\gamma)^2}{\delta}\right)\\
    &\leq eU_T\ln\left(\frac{2\ln(eU_T)^2}{\delta}\right)+4\sqrt{V_T\ln\left(\frac{\ln^2(V_T)}{\delta}\right)} \tag{choosing $\gamma=\sqrt{\frac{V_T}{\ln(2\delta^{-1}\ln^2(V_T)}}$}\,.
\end{align*}
The reasoning for the second inequality is the following. We are computing $\min_{i\geq i_0}f(i)+g(i)$, where $f(i)$ is monotonically decreasing in $i$ and $g(i)$ is monotonically increasing. Let $i^*=\argmin_{i\in \mathbb Z} f(i)+g(i)$.
If $i^* \geq i_0$, the equation is obviously true. Otherwise the optimal solution is $\min_{i\geq i_0}f(i)+g(i)=f(i_0)+g(i_0)\leq f(i^*)+g(i_0)$ by the monotonicity.

\end{proof}

\begin{lemma}
\label{lem: freedman over ball}
Let $X_1,\dots \in \mathbb{R}^S$ be a sequence of i.i.d.\ random vectors with mean $\mu$, variance $\Sigma$ such that $\norm{X_i}_1< c$ almost surely. Then with probability at least $1-2S\delta$, it holds for all $w\in\mathbb{R}^S$ such that $\norm{w}_{\infty}<C$ and $T>0$ simultaneously:
\begin{align*}
\sum_{t=1}^T \inner{X_t,w} \leq   4\sqrt{STw^\top\Sigma w \left(\ln(\delta^{-1})+2\ln\ln(Tc^2)\right)} + eScC\ln\left(\frac{2\ln^2(ecC)}{\delta}\right)\,.
\end{align*}
\end{lemma}
\begin{proof}
The proof is extends \pref{lem: freedman}.
Let $v_1,\dots,v_S$ be the eigenvectors of $\Sigma$, then we have with probability $1-2S\delta$ for all $v_i$ simultaneously
\begin{align*}
    \left|\sum_{t=1}^T\inner{X_t,v_i}\right| \leq 4\sqrt{Tv_i^\top\Sigma v_i \left(\ln(\delta^{-1})+2\ln\ln(Tc^2)\right)} + ec\ln\left(\frac{2\ln^2(ec)}{\delta}\right)\,.
\end{align*}
Let $w = \sum_{i=1}^Sa_iv_i$, where we know that $\norm{w}_2 \leq \sqrt{S}C$ and $\sum_{i=1}^S|a_i|\leq SC$.
This implies
\begin{align*}
    \sum_{t=1}^T\inner{X_t,w} &\leq \sum_{i=1}^S4\sqrt{Ta_i^2v_i^\top\Sigma v_i \left(\ln(\delta^{-1})+2\ln\ln(Tc^2)\right)} + |a_i|ec\ln\left(\frac{2\ln^2(ec)}{\delta}\right)\\
    &\leq 4\sqrt{STw^\top \Sigma w\left(\ln(\delta^{-1})+2\ln\ln(Tc^2)\right)} +ecCS\ln\left(\frac{2\ln^2(ec)}{\delta}\right)\,.
\end{align*}
\end{proof}

\begin{lemma}
\label{lem: lnln concentration}
Let $X_1,\dots \in [0,B]$ be a sequence of i.i.d.\ random variables with mean $\mu$ and variance $\sigma^2$. Then with probability at least $1-\delta$, it holds for all $T\geq 8\ln(2\delta^{-1})$ simultaneously:
\begin{align*}
    \sum_{t=1}^TX_t - T\mu \leq \sqrt{8 \sum_{i=1}^T\left(X_i-\frac{1}{T}\sum_{j=1}^TX_j\right)^2(\ln(2\delta^{-1})+4\ln\ln(T))} + 12B(\ln(2\delta^{-1})+4\ln\ln(T)). 
\end{align*}
\end{lemma}
\begin{proof}
By \pref{lem: freedman}, we have that with probability $1-\delta/2$, for all $T>0$ simultaneously, 
\begin{align*}
    \sum_{t=1}^TX_t - T\mu \leq \sqrt{ T\sigma^2(\ln(2\delta^{-1})+4\ln\ln(T))} + B(\ln(2\delta^{-1})+4\ln\ln(T))\,.
\end{align*}
Applying the same Lemma to the sequence $Z_i = -(X_i-\mu)^2+\sigma^2$, we have  with probability $1-\delta/2$ for all $T>0$ simultaneously
\begin{align*}
    &T\sigma^2-\sum_{t=1}^T(X_t-\mu)^2 \\ 
    &\leq \sqrt{ \sum_{t=1}^T\E_{t-1}[((X_t-\mu)^2-\sigma^2)^2](\ln(2\delta^{-1})+4\ln\ln(T))} + \sigma^2(\ln(2\delta^{-1})+4\ln\ln(T))\\
    &\leq B\sqrt{ T\sigma^2(\ln(2\delta^{-1})+4\ln\ln(T))} + \sigma^2(\ln(2\delta^{-1})+4\ln\ln(T))
    \\
    &\leq \frac{T}{2}\sigma^2+5B^2(\ln(2\delta^{-1})+4\ln\ln(T))
    \,.
\end{align*}
Finally we have
\begin{align*}
    \sum_{t=1}^T(X_t-\mu)^2 = \sum_{t=1}^T\left(X_t-\frac{1}{T}\sum_{j=1}^TX_j\right)^2 +T\left(\mu-\frac{1}{T}\sum_{t=1}^TX_t\right)^2\,.
\end{align*}
Combining everything and taking a union bounds, leads to with probability $1-\delta$
\begin{align*}
    &\sum_{t=1}^TX_t - T\mu \\
    &\leq \sqrt{2\left(\sum_{t=1}^T\left(X_t-\frac{1}{T}\sum_{j=1}^TX_j\right)^2+\frac{1}{T}\left(\sum_{t=1}^TX_t - T\mu\right)^2+10B^2\right)(\ln(2\delta^{-1})+4\ln\ln(T))
    }\\
    &\qquad\qquad+B(\ln(2\delta^{-1})+4\ln\ln(T))\\
    &\leq \left|\sum_{t=1}^TX_t-T\mu\right|\sqrt{\frac{2\ln(2\delta^{-1})+4\ln\ln(T)}{T}} +\sqrt{2\sum_{t=1}^T\left(X_t-\frac{1}{T}\sum_{j=1}^TX_j\right)^2}+6B(\ln(2\delta^{-1})+4\ln\ln(T)).
\end{align*}
For $T \geq 8\ln(2\delta^{-1})$, we have $\frac{2\ln(2\delta^{-1})+4\ln\ln(T)}{T}\leq \frac{1}{4}$ and rearranging finishes the proof. 
\end{proof}

\begin{lemma}[Lemma 30, \cite{chen2021implicit}]\label{lem: lemma 30} Let $\|X\|_\infty\leq C$, then $\mathbb{V}(P,X^2)\leq 4C\mathbb{V}(P,X)$ for any $P$.  

\end{lemma}



%% file: appendix-nonproper.tex
\section{Weakening the assumption on proper policies}
\label{app:nonproper}
\pref{assum: proper} can be weakened to the following: 
\begin{assumption}\label{assum: weakened assumption}
    There exists a policy $\pi^\star \in \piSD$ such that 
    \begin{itemize}
        \item $V^{\pi^\star}(s)\geq V^{\pi}(s)$ for all $s\in\calS$ and $\pi\in\piHD$. 
        \item $T_\star \triangleq \max_s \E^{\pi^\star}[\tau~|~s_1=s] < \infty$, where $\tau$ is the time index right before reaching $g$. 
    \end{itemize}
\end{assumption}
In words, \pref{assum: weakened assumption} assumes that there exists an optimal policy that is stationary and proper. 
If such an optimal policy is not unique, we can take $T_\star$ to be the minimum over all such policies. Notice that the first part of \pref{assum: weakened assumption} is sufficient for all our algorithms to work, though the regret bound has a $\ln T$ factor, which could be unbounded. That is why in the main text we introduced the stronger  \pref{assum: proper} and upper bound $T$ by the order of $KT_{\max}$. Below we show that with the additional second part of \pref{assum: weakened assumption} and the algorithmic trick introduced in \cite{tarbouriech2021stochastic}, the $T_{\max}$ dependency can be replaced by $T_\star$. 

Assuming that an order optimal bound of $\Tstar$ is known, the agent modifies the MDP by modifying all rewards $\tilde r(s,a) = r(s,a)-\frac{1}{K\Tstar}$.
The value of the optimal policy in the modified MDP is smaller than that in the original MDP by at most $\frac{1}{K}$. It is then sufficient to bound the regret for the modified MDP, since
\begin{align*}
    \sum_{k=1}^K\left(V^\star(s_\init)-\sum_{t=t_k}^{e_k}r(s_t, a_t)\right)\leq 1+\sum_{k=1}^K\left({\tilde V}^\star(s_\init)-\sum_{t=t_k}^{e_k}\tilde{r}(s_t, a_t)\right)\,.
\end{align*}

For the modified MDP, we can show that the total time horizon $T$ is bounded. By a trivial bound of $\mathbb{V}(P_t,V^\star)\leq \Bstar^2$, combined with \pref{lem: V*-Q*}, we get that with probability at least $1-\delta$, 
\begin{align*}
    \sum_{k=1}^K \left( {\tilde V}^\star(s_\init) - \sum_{t=t_k}^{e_k}\tilde r(s_t, a_t)\right) \leq  O\left( \sqrt{SA\Bstar^2 T\tilde{\iota}_{T,B,\delta}} + B S^2 A\tilde{\iota}_{T,B,\delta}\right)\,. 
\end{align*}
Notice that $K\tilde{V}^\star(s_\init)\geq KV^\star(s_{\init}) - 1$ and $\sum_{t=1}^{T} \tilde{r}(s_t,a_t) =  -\frac{T}{KT_\star} + \sum_{t=1}^T r(s_t,a_t)\leq -\frac{T}{KT_\star} + KV^\star(s_{\init}) + R\sqrt{K\ln(KR/\delta)} + \Rmax \ln (K\Rmax/\delta)$ with probability $\geq 1-O(\delta)$ by \pref{lem: freedman} (notice that $\E\big[\sum_{t=1}^T r(s_t,a_t) \big]\leq KV^\star(s_{\init})$). 
Rearranging by $T$ leads to with probability at least $1-O(\delta)$,
\begin{align*}
    T \leq  O\left(\poly(S,A,B,\Bstar,R,R_{\max}, K,T_\star, \delta^{-1})\right)\,.
\end{align*} 
This upper bound on $T$ helps us to replace the $\ln T$ dependency in the regret by a log term that only involves algorithm-independent quantities. 

If an order optimal bound of $\Tstar$ is not known, we can follow the arguments in \citet{tarbouriech2021stochastic} that sets $\tilde{r}(s,a)=r(s,a) - \frac{1}{K^{n}}$ for some $n\gg 1$. By this, we can also remove the dependency on $T_{\max}$, with the price of an additional $\frac{\Tstar}{K^{n-1}}$ regret. 

%% file: alt2023-sample.bbl
\begin{thebibliography}{27}
\providecommand{\natexlab}[1]{#1}
\providecommand{\url}[1]{\texttt{#1}}
\expandafter\ifx\csname urlstyle\endcsname\relax
  \providecommand{\doi}[1]{doi: #1}\else
  \providecommand{\doi}{doi: \begingroup \urlstyle{rm}\Url}\fi

\bibitem[Auer et~al.(2002)Auer, Cesa-Bianchi, Freund, and
  Schapire]{auer2002nonstochastic}
Peter Auer, Nicolo Cesa-Bianchi, Yoav Freund, and Robert~E Schapire.
\newblock The nonstochastic multiarmed bandit problem.
\newblock \emph{SIAM journal on computing}, 32\penalty0 (1):\penalty0 48--77,
  2002.

\bibitem[Azar et~al.(2017)Azar, Osband, and Munos]{azar2017minimax}
Mohammad~Gheshlaghi Azar, Ian Osband, and R{\'e}mi Munos.
\newblock Minimax regret bounds for reinforcement learning.
\newblock In \emph{International Conference on Machine Learning}, pages
  263--272. PMLR, 2017.

\bibitem[Chen and Luo(2021)]{chen2021finding}
Liyu Chen and Haipeng Luo.
\newblock Finding the stochastic shortest path with low regret: The adversarial
  cost and unknown transition case.
\newblock In \emph{International Conference on Machine Learning}, pages
  1651--1660. PMLR, 2021.

\bibitem[Chen and Luo(2022)]{chen2022near}
Liyu Chen and Haipeng Luo.
\newblock Near-optimal goal-oriented reinforcement learning in non-stationary
  environments.
\newblock \emph{Advances in Neural Information Processing Systems}, 2022.

\bibitem[Chen et~al.(2021{\natexlab{a}})Chen, Jafarnia-Jahromi, Jain, and
  Luo]{chen2021implicit}
Liyu Chen, Mehdi Jafarnia-Jahromi, Rahul Jain, and Haipeng Luo.
\newblock Implicit finite-horizon approximation and efficient optimal
  algorithms for stochastic shortest path.
\newblock \emph{Advances in Neural Information Processing Systems},
  34:\penalty0 10849--10861, 2021{\natexlab{a}}.

\bibitem[Chen et~al.(2021{\natexlab{b}})Chen, Luo, and Wei]{chen2021minimax}
Liyu Chen, Haipeng Luo, and Chen-Yu Wei.
\newblock Minimax regret for stochastic shortest path with adversarial costs
  and known transition.
\newblock In \emph{Conference on Learning Theory}, pages 1180--1215. PMLR,
  2021{\natexlab{b}}.

\bibitem[Chen et~al.(2022{\natexlab{a}})Chen, Jain, and Luo]{chen2022improved}
Liyu Chen, Rahul Jain, and Haipeng Luo.
\newblock Improved no-regret algorithms for stochastic shortest path with
  linear mdp.
\newblock In \emph{International Conference on Machine Learning}, pages
  3204--3245. PMLR, 2022{\natexlab{a}}.

\bibitem[Chen et~al.(2022{\natexlab{b}})Chen, Luo, and
  Rosenberg]{chen2022policy}
Liyu Chen, Haipeng Luo, and Aviv Rosenberg.
\newblock Policy optimization for stochastic shortest path.
\newblock \emph{Conference on Learning Theory}, 2022{\natexlab{b}}.

\bibitem[Cohen et~al.(2021)Cohen, Efroni, Mansour, and
  Rosenberg]{cohen2021minimax}
Alon Cohen, Yonathan Efroni, Yishay Mansour, and Aviv Rosenberg.
\newblock Minimax regret for stochastic shortest path.
\newblock \emph{Advances in Neural Information Processing Systems},
  34:\penalty0 28350--28361, 2021.

\bibitem[Dann and Brunskill(2015)]{dann2015sample}
Christoph Dann and Emma Brunskill.
\newblock Sample complexity of episodic fixed-horizon reinforcement learning.
\newblock In \emph{Advances in Neural Information Processing Systems}, pages
  2818--2826, 2015.

\bibitem[Efroni et~al.(2021)Efroni, Merlis, and
  Mannor]{efroni2020reinforcement}
Yonathan Efroni, Nadav Merlis, and Shie Mannor.
\newblock Reinforcement learning with trajectory feedback.
\newblock In \emph{AAAI Conference on Artificial Intelligence}, 2021.

\bibitem[Jafarnia-Jahromi et~al.(2021)Jafarnia-Jahromi, Chen, Jain, and
  Luo]{jafarnia2021online}
Mehdi Jafarnia-Jahromi, Liyu Chen, Rahul Jain, and Haipeng Luo.
\newblock Online learning for stochastic shortest path model via posterior
  sampling.
\newblock \emph{arXiv preprint arXiv:2106.05335}, 2021.

\bibitem[Jin et~al.(2018)Jin, Allen-Zhu, Bubeck, and Jordan]{jin2018q}
Chi Jin, Zeyuan Allen-Zhu, Sebastien Bubeck, and Michael~I Jordan.
\newblock Is q-learning provably efficient?
\newblock In \emph{Conference on Neural Information Processing Systems}, 2018.

\bibitem[Kakade and Langford(2002)]{kakade2002approximately}
Sham Kakade and John Langford.
\newblock Approximately optimal approximate reinforcement learning.
\newblock In \emph{In Proc. 19th International Conference on Machine Learning}.
  Citeseer, 2002.

\bibitem[Lattimore and Szepesv{\'a}ri(2020)]{lattimore2020bandit}
Tor Lattimore and Csaba Szepesv{\'a}ri.
\newblock \emph{Bandit algorithms}.
\newblock Cambridge University Press, 2020.

\bibitem[Min et~al.(2022)Min, He, Wang, and Gu]{min2022learning}
Yifei Min, Jiafan He, Tianhao Wang, and Quanquan Gu.
\newblock Learning stochastic shortest path with linear function approximation.
\newblock In \emph{International Conference on Machine Learning}, pages
  15584--15629. PMLR, 2022.

\bibitem[Puterman(2014)]{puterman2014markov}
Martin~L Puterman.
\newblock \emph{Markov decision processes: discrete stochastic dynamic
  programming}.
\newblock John Wiley \& Sons, 2014.

\bibitem[Rosenberg et~al.(2020)Rosenberg, Cohen, Mansour, and
  Kaplan]{rosenberg2020near}
Aviv Rosenberg, Alon Cohen, Yishay Mansour, and Haim Kaplan.
\newblock Near-optimal regret bounds for stochastic shortest path.
\newblock In \emph{International Conference on Machine Learning}, pages
  8210--8219. PMLR, 2020.

\bibitem[Tarbouriech et~al.(2020)Tarbouriech, Garcelon, Valko, Pirotta, and
  Lazaric]{tarbouriech2020no}
Jean Tarbouriech, Evrard Garcelon, Michal Valko, Matteo Pirotta, and Alessandro
  Lazaric.
\newblock No-regret exploration in goal-oriented reinforcement learning.
\newblock In \emph{International Conference on Machine Learning}, pages
  9428--9437. PMLR, 2020.

\bibitem[Tarbouriech et~al.(2021{\natexlab{a}})Tarbouriech, Pirotta, Valko, and
  Lazaric]{tarbouriech2021sample}
Jean Tarbouriech, Matteo Pirotta, Michal Valko, and Alessandro Lazaric.
\newblock Sample complexity bounds for stochastic shortest path with a
  generative model.
\newblock In \emph{Algorithmic Learning Theory}, pages 1157--1178. PMLR,
  2021{\natexlab{a}}.

\bibitem[Tarbouriech et~al.(2021{\natexlab{b}})Tarbouriech, Zhou, Du, Pirotta,
  Valko, and Lazaric]{tarbouriech2021stochastic}
Jean Tarbouriech, Runlong Zhou, Simon~S Du, Matteo Pirotta, Michal Valko, and
  Alessandro Lazaric.
\newblock Stochastic shortest path: Minimax, parameter-free and towards
  horizon-free regret.
\newblock \emph{Advances in Neural Information Processing Systems},
  34:\penalty0 6843--6855, 2021{\natexlab{b}}.

\bibitem[Vial et~al.(2022)Vial, Parulekar, Shakkottai, and
  Srikant]{vial2022regret}
Daniel Vial, Advait Parulekar, Sanjay Shakkottai, and R~Srikant.
\newblock Regret bounds for stochastic shortest path problems with linear
  function approximation.
\newblock In \emph{International Conference on Machine Learning}, pages
  22203--22233. PMLR, 2022.

\bibitem[Wang et~al.(2020)Wang, Salakhutdinov, and Yang]{wang2020reinforcement}
Ruosong Wang, Ruslan Salakhutdinov, and Lin~F Yang.
\newblock Reinforcement learning with general value function approximation:
  Provably efficient approach via bounded eluder dimension.
\newblock In \emph{Conference on Neural Information Processing Systems}, 2020.

\bibitem[Zanette and Brunskill(2019)]{zanette2019tighter}
Andrea Zanette and Emma Brunskill.
\newblock Tighter problem-dependent regret bounds in reinforcement learning
  without domain knowledge using value function bounds.
\newblock In \emph{International Conference on Machine Learning}, 2019.

\bibitem[Zhang et~al.(2020)Zhang, Zhou, and Ji]{zhang2020almost}
Zihan Zhang, Yuan Zhou, and Xiangyang Ji.
\newblock Almost optimal model-free reinforcement learningvia
  reference-advantage decomposition.
\newblock \emph{Advances in Neural Information Processing Systems},
  33:\penalty0 15198--15207, 2020.

\bibitem[Zhang et~al.(2021)Zhang, Ji, and Du]{zhang2021reinforcement}
Zihan Zhang, Xiangyang Ji, and Simon Du.
\newblock Is reinforcement learning more difficult than bandits? a near-optimal
  algorithm escaping the curse of horizon.
\newblock In \emph{Conference on Learning Theory}, pages 4528--4531. PMLR,
  2021.

\bibitem[Zhang et~al.(2022)Zhang, Ji, and Du]{zhang2022horizon}
Zihan Zhang, Xiangyang Ji, and Simon Du.
\newblock Horizon-free reinforcement learning in polynomial time: the power of
  stationary policies.
\newblock In \emph{Conference on Learning Theory}, pages 3858--3904. PMLR,
  2022.

\end{thebibliography}
